\documentclass{article}

\usepackage[utf8]{inputenc} 
\usepackage[T1]{fontenc}    
\usepackage{hyperref}       
\usepackage{url}            
\usepackage{booktabs}       
\usepackage{amsfonts}       
\usepackage{nicefrac}       
\usepackage{microtype}      
\usepackage[alpha]{mdpn}
\usepackage{amsmath}
\usepackage{array,tabularx}
\usepackage{thm-restate}
\usepackage{amsthm}	
\newtheorem{definition}{Definition}[]

\renewcommand{\vec}[1]{\mathbf{#1}}
\usepackage{subfigure}

\usepackage{url}            
\usepackage{booktabs}       
\usepackage{amsfonts}       
\usepackage{nicefrac}       
\usepackage{microtype}      
\usepackage{amsmath}
\usepackage{mathtools}  

\usepackage{algorithm}
\usepackage{algorithmic}

\usepackage{hyperref}

\usepackage{natbib}

\usepackage{theoremref} 
\usepackage{amsxtra} 

\usepackage{textcase}

\newcommand{\pib}{{\pi_b}}

\newcommand{\pieval}{{\pi_e}}

\newcommand{\bx}{{\boldsymbol x}}
\newcommand{\bw}{{\boldsymbol w}}

\newcommand{\data}{{\mathcal{D}}}

\DeclareMathOperator*{\argmax}{\arg\!\max}

\newcommand\numberthis{\addtocounter{equation}{1}\tag{\theequation}}

\usepackage{algorithmic}
\usepackage{enumitem}

\usepackage[accepted]{icml2020}

\begin{document}

\setlength{\abovedisplayskip}{5pt}
\setlength{\belowdisplayskip}{5pt}

\twocolumn[
\icmltitle{Reducing Sampling Error in Batch Temporal Difference
Learning}



\icmlsetsymbol{equal}{*}

\begin{icmlauthorlist}
\icmlauthor{Brahma S. Pavse}{ut}
\icmlauthor{Ishan Durugkar}{ut}
\icmlauthor{Josiah P. Hanna}{edin,wisc}
\icmlauthor{Peter Stone}{ut,sony}
\end{icmlauthorlist}

\icmlaffiliation{ut}{The University of Texas at Austin}
\icmlaffiliation{edin}{School of Informatics, University of Edinburgh}
\icmlaffiliation{sony}{Sony AI}
\icmlaffiliation{wisc}{To be joining the Computer Sciences Department, University of Wisconsin--Madison}

\icmlcorrespondingauthor{Brahma S. Pavse}{brahmasp@cs.utexas.edu}

\icmlkeywords{batch reinforcement learning,
temporal difference learning, importance sampling, off-policy, certainty-equivalence, sampling error,
data efficiency}

\vskip 0.3in
]



\printAffiliationsAndNotice{} 

\begin{abstract}
Temporal difference (TD) learning is one of the main foundations of modern reinforcement learning. This paper studies the use of TD(0), a canonical TD algorithm, to estimate the value function of a given policy from a batch of data. In this batch setting, we show that TD(0) may converge to an inaccurate value function because the update following an action is weighted according to the number of times that action occurred in the batch -- not the true probability of the action under the given policy. To address this limitation, we introduce \textit{policy sampling error corrected}-TD(0) (PSEC-TD(0)). PSEC-TD(0) first estimates the empirical distribution of actions in each state in the batch and then uses importance sampling to correct for the mismatch between the empirical weighting and the correct weighting for updates following each action. We refine the concept of a certainty-equivalence estimate and argue that PSEC-TD(0) is a more data efficient estimator than TD(0) for a fixed batch of data. Finally, we conduct an empirical evaluation of PSEC-TD(0) on three batch value function learning tasks, with a hyperparameter sensitivity analysis, and show that PSEC-TD(0) produces value function estimates with lower mean squared error than TD(0).
\end{abstract}

\section{Introduction}
Reinforcement
learning (RL) \cite{Sutton:1998:IRL:551283} algorithms have
been applied to a variety of
sequential-decision making problems such as robot
manipulation \cite{rlroboticssurvey, DBLP:journals/corr/GuHLL16} and autonomous driving \cite{sallab2017deep}.
Many RL algorithms 
learn an optimal control policy by estimating
the \textit{value function}, a function
that gives the expected return 
from each state 
when following a particular policy \cite{10.2307/2630487, 10.5555/26970, NIPS1999_1786}.
These algorithms
require accurate value function estimation with finite data.
A fundamental approach to value function learning is the 
temporal difference (TD) algorithm \cite{sutton1988learning}.

In this work, we focus on improving the accuracy of the value function learned by \textit{batch} TD, where TD updates for a value function are computed from a fixed batch of  data. We show that batch TD(0) may converge to an inaccurate value function since it ignores the known action probabilities of the policy it is evaluating.
For example, consider a single state in which the evaluation policy selects between action $a_1$ or $a_2$ with probability $0.5$.
If, in the finite batch of observed data, $a_1$ actually happens to occur twice as often as $a_2$ then TD updates following $a_1$ will receive twice as much weight as updates following $a_2$, even though in expectation they should receive the same weight.
We describe this finite-sample error in the value function estimate as \textit{policy sampling error}. To correct for policy sampling
error we propose to first estimate the maximum likelihood policy from the observed data and then use importance sampling \cite{precup2000eligibility} to account for the mismatch between the frequency of sampled actions and their true probability under the evaluation policy.
Variants of this technique have
been successful in multi-armed bandits \cite{Li2015TowardMO, DBLP:journals/corr/abs-1809-03084, DBLP:journals/corr/abs-1808-00232}, policy evaluation \cite{ICML2019-Hanna}, and policy gradient learning \cite{hanna2019reducing}.
However, we are the first to show that this technique can be used to 
correct for policy sampling error in value function estimation and the first to show the benefit of importance sampling in on-policy value function estimation. We show that by using the available policy information, our approach is more
data efficient than vanilla batch TD(0).
We call our new value function learning algorithm batch \textit{policy sampling error corrected-TD(0) (PSEC-TD(0))}. 

The contributions of the paper are the following:
\begin{enumerate}
    \itemsep0em
    \item Show that the fixed point that batch TD(0) converges to for a given policy is inaccurate with respect to the true value function.
    \item Introduce the batch PSEC-TD(0) algorithm that reduces the policy sampling error in batch TD(0).
    \item Refine the concept of a certainty-equivalence estimate for TD(0) \cite{sutton1988learning} and provide theoretical justification that batch PSEC-TD(0) is more data efficient than batch TD(0).
    \item Empirically analyze batch PSEC-TD(0) in the tabular and function approximation setting.
\end{enumerate}

\section{Background}
This section introduces notation and formally specifies the batch value function learning problem.

\subsection{Notation and Definitions}
Following the standard MDPNv1 notation \cite{DBLP:journals/corr/Thomas15a}, we consider a Markov decision process (MDP) with state space $\sset$, action space $\aset$, reward function $R$, transition dynamics function $P$, and discount factor $\gamma$ \cite{puterman2014markov}.
In any state $s$, an agent selects stochastic actions according to a policy $\pi$, $a \sim \pi(\cdot | s)$.
After taking an action $a$ in state $s$ the agent transitions to a new state $s' \sim P(\cdot | s,a)$ and receives reward $R(s, a, s')$.
We assume $\sset$ and $\aset$ to be finite;
however, our experiments also consider infinite sized $\sset$ and $\aset$. We consider the episodic, discounted, and finite 
horizon setting.
The policy and MDP jointly induce a \textit{Markov 
reward process} (MRP), in which the agent 
transitions between states $s$ and $s'$ with 
probability $P(s' | s)$ and receives reward 
$R(s,s')$. Finally, $\bx(s): \sset \rightarrow 
\mathbb{R}^d$ gives a column feature vector for 
each state $s \in \sset$.

We are concerned with computing
the value function, $v^{\pi}: \sset \rightarrow \mathbb{R}$, that gives the \textit{value} of any state.
The value of a particular state is the expected discounted return,  
i.e. the expected sum of discounted
rewards when following policy $\pi$ from that state:
{\small
\begin{equation}\footnotesize
v^{\pi}(s) \coloneqq \mathbf{E}_{\pi} \biggl [\sum_{k=0}^{L} \gamma^k R_{t+k+1} \biggm | s_t = s \biggr], \forall s \in \sset
\end{equation}
}
where $L$ is the terminal time-step and the expectation is taken over the distribution of future states, actions, and rewards under $\pi$ and $P$.

\subsection{Batch Value Prediction}
This work investigates the problem of approximating $v^{\pi_e}$ given a batch of data, $\mathcal{D}$, and an evaluation policy,  $\pi_e$.
Let a single episode, $\tau$, be defined
as $\tau := (s_0, a_0, r_{0}, s_{1},
...,$ $s_{L_{\tau-1}}, a_{L_{\tau-1}}, r_{L_{\tau-1}})$, where $L_\tau$ is the length
of the episode $\tau$. 
%
The batch of data consists
of $m$ episodes, i.e.,  $\mathcal{D} := \{\tau_i\}_{i = 0}^{m-1}$. The policy that generated the batch of data is
called the \textit{behavior policy}, $\pi_b$.
If $\pi_b$ is the same as $\pi_e$ for all episodes then learning is said to be done \textit{on-policy}; otherwise it is \textit{off-policy}.

In batch value prediction,
a value function learning algorithm uses a fixed batch of
data to learn an estimate $\hat{v}^{\pi_e}$ that approximates the true value function, $v^{\pi_e}$.
In this work, we introduce algorithmic and theoretical concepts with the linear approximation
of $v^{\pi_e}$:
\begin{align*}\footnotesize
\hat{v}^{\pi_e}(s) := \bw^T \bx(s)
\end{align*}
thus, in the linear case, we seek to find a weight vector $\bw$, such that $\bw^T \bx(s)$ approximates the true value, $v^{\pi_e}(s)$. However, our empirical
study also considers the non-linear approximation of $v^{\pi_e}$.
The error of the
predicted value function, $\hat{v}^{\pi_e}$, with respect to the true value function, $v^{\pi_e}$, is
measured by calculating the mean squared value error 
between  $v^{\pi_e}(s) \text{ and } \hat{v}^{\pi_e}(s)$ 
$\forall s \in \sset$ weighted by the proportion of 
time spent in each state under policy $\pieval$, 
$d_{\pieval}(s)$.
Thus, we seek to find a weight vector $\bw$ that minimizes:
{\small
\begin{equation}\footnotesize
\label{eq:mse_objective}
\operatorname{MSVE}(\bw) \coloneqq \sum_{s \in \sset} d_{\pieval}(s) \biggl (v^{\pi_e}(s) -  \bw^T \bx(s)\biggr)^2
\end{equation}
}
In this work, we compare data efficiency between two algorithms, $X$ and $Y$,
 as follows:
\begin{definition}{Data Efficiency.}
A prediction algorithm $X$ is more data efficient than algorithm $Y$ if estimates from $X$ have, on average, lower MSVE than estimates from $Y$ for a given batch size.
\end{definition}

\subsection{Batch Linear TD(0)}

A fundamental algorithm for value prediction is the single-step temporal difference learning algorithm, TD(0).
Algorithm \ref{algo:batch_linear_td} gives pseudo-code for the batch linear TD(0) algorithm described by \citet{sutton1988learning}.

 \begin{algorithm}[H]
  \caption{Batch Linear TD(0) to estimate $v^{\pi_e}$}
  \label{algo:batch_linear_td}
  \begin{algorithmic}[1]
    \STATE Input: policy to evaluate $\pi_e$, behavior policy $\pi_b$, batch $\mathcal{D}$, linear value function, $\hat{v}: \sset \times \mathbb{R}^d \rightarrow \mathbb{R}$, step-size $\alpha > 0$, convergence
    threshold $\Delta > 0$
    \STATE Initialize: weight vector $\bw_0$ arbitrarily (e.g.: $\bw_0 := \vec{0}$), aggregation vector $\bm{u} := \vec{0}$, batch process
    counter, $i = 0$
    \WHILE{$|\bw_{i + 1} - \bw_{i}| \geq \vec{1}\cdot\Delta$} 
	    \FOR{each episode, $\tau \in \mathcal{D}$}
	        \FOR{each transition, $(s,
	        a, r, s') \in \tau$}
                \STATE $\hat{y} \leftarrow r + \gamma \bw_i^T \bx(s')$\\
                \STATE $\rho \leftarrow \frac{\pi_e(a|s)}{\pi_b(a|s)}$ \COMMENT{for on-policy, $\pi_b = \pi_e$}
                \STATE $\bm{u} \leftarrow \bm{u} + \left[\rho \hat{y} -  \bw_i^T\bx(s)\right]\bx(s)$
            \ENDFOR
        \ENDFOR
        \STATE $\bw_{i + 1} \leftarrow \bw_{i} + \alpha\bm{u}$ \COMMENT{batch update}
        \STATE $\bm{u} \leftarrow \vec{0}$ \COMMENT {clear aggregation}
        \STATE $i \leftarrow i + 1$
	\ENDWHILE
  \end{algorithmic}
\end{algorithm}
\citet{sutton1988learning} proved that batch linear TD(0) converges to a fixed point
in the on-policy case i.e. when $\pi_e = \pib$.
%
An off-policy batch TD(0) algorithm uses 
importance sampling ratios to ensure that the expected update is the same as it would be if actions were taken with $\pi_e$ instead of $\pi_b$ 
\citep{precup2000eligibility}.
%
%
Unlike on-policy TD(0), off-policy TD(0) is \textit{not} 
guaranteed to converge \citep{baird1995residual}.

\section{Convergence of Batch Linear TD(0)}
\label{sec:td_convergence}

In this section, we discuss the convergence of  batch linear TD(0) to a fixed-point, the \textit{certainty equivalence estimate} (CEE) for the underlying Markov reward process (MRP).
We refine this concept to better reflect our objective of evaluating a policy in an MDP and then prove that batch TD(0) converges to an equivalent fixed point that ignores knowledge of the known evaluation policy, $\pieval$, leading to inaccuracy in the value function estimate.
This result motivates our proposed algorithm.

First, we introduce additional notation and assumptions.
In this section, we assume that we are in the on-policy setting ($\pib = \pieval$).
Let $\widehat{\sset}$ be the set of states and $\widehat{\aset}$ be the set of actions that appear in $\data$ and let $\bar{R}(s)$ be the mean reward received when transitioning from state $s$ in the batch $\data$.
Finally, if the notation includes a hat ( $\hat{}$  ), it is the maximum-likelihood estimate (MLE) according to $\data$. For
example, $\hat{\pi}$ is the MLE of $\pib$.
 \citet{sutton1988learning} proved that batch linear TD(0)
converges to the \textit{CEE}.
That is, it converges to the exact value function of the maximum likelihood MRP according to the observed batch.
This exact value function can be calculated using dynamic programming \citep{10.5555/862270, 10.5555/26970} with the MLE MRP transition function.
We call this value function estimate the \textit{Markov reward process certainty equivalence estimate} (MRP-CEE).
%
%
%
\begin{definition}{Markov Reward Process Certainty Equivalence Estimate
(MRP-CEE) Value Function.}
\label{def:mrp_cee}
The MRP-CEE is the value function $\hat{v}_\mathtt{MRP}$ that, $\forall s, s'\in \widehat{\sset}$, satisfies:
{\small
\begin{equation}\label{eq:mrp-dp}\footnotesize
    \hat{v}_\mathtt{MRP}(s) = \bar{R}(s) + \gamma \sum_{k \in \widehat{\sset}} \widehat{P}(s'|s) \hat{v}_\mathtt{MRP}(s').
\end{equation}
}
\end{definition}

Having now defined the MRP-CEE value function, we prove that batch TD(0) converges to the MRP-CEE value function.
This fact was first proven by  \citet{sutton1988learning} (see Theorem 3 of \citet{sutton1988learning}), however the original proof only considers rewards upon termination and no discounting.
The extension to rewards per-step and discounting is straightforward, but to the best of our knowledge has not appeared in the literature before.
Following Sutton's proof \cite{sutton1988learning}, we first prove the extension before
extending the proof to an MDP, where the data inefficiency of TD(0) becomes clear. Proof details are in Appendix 
\ref{sec:mrp_td_convergence_proof}.

\begin{restatable}[Batch Linear TD(0)
Convergence]{theorem}{thmrptdconvergence}
\label{th:mrp_td_convergence}
For any batch whose observation vectors $\{ \bx(s) | s \in \widehat{\mathcal{S}}\}$ are linearly independent, there exists an 
$\epsilon > 0$  such that, for all positive $\alpha < \epsilon$ and for any 
initial
weight vector, the predictions for linear TD(0) converge under
repeated presentations of the batch with weight updates after
each complete presentation to the fixed-point (\ref{eq:mrp-dp}).
\end{restatable}
In RL, the transitions of an MRP are a function of the behavior policy \emph{and} transition dynamics distributions.
That is $\forall s, s' \in \widehat{\sset}$:
{\small
\begin{align*}
    \begin{split}
        \widehat{P}(s'|s) &= \sum_{a \in \widehat{\aset}} \hat{\pi}(a|s)\widehat{P}(s'|s, a),\\
    \bar{R}(s) &= \sum_{a \in \widehat{\aset}} \hat{\pi}(a|s)\bar{R}(s, a) 
    \end{split}
\end{align*}
}
where $\bar{R}(s, a)$ is the mean reward observed in state $s$ on taking action $a$.
We define a new certainty-equivalence estimate that separates these two factors.
%
We call this new value function estimate the \textit{Markov decision process certainty equivalent estimate} (MDP-CEE).
\begin{definition}{Markov Decision Process Certainty Equivalence Estimate (MDP-CEE) Value Function.}
\label{def:mdp_cee}
The MDP-CEE is the value function, $\hat{v}_\mathtt{MDP}^{\hat{\pi}}$, that, $\forall s, s' \in \widehat{\sset}$, satisfies:
{\small
\begin{align}\label{eq:mdp-dp}\footnotesize
\begin{split}
    \hat{v}_\mathtt{MDP}^{\hat{\pi}}(s) = \sum_{a \in \widehat{\aset}} \hat{\pi}(a|s)\left( \bar{R}(s, a) +  \gamma \sum_{s' \in \widehat{\sset}}  \widehat{P}(s'|s, a) \hat{v}_\mathtt{MDP}^{\hat{\pi}}(s')\right)
\end{split}
\end{align}
}
\end{definition}
Given the definitions of $\widehat{P}$ and $\bar{R}$, the MRP-CEE (Definition \ref{def:mrp_cee}) and MDP-CEE (Definition \ref{def:mdp_cee}) are equivalent. Theorem \ref{th:mdp_td_convergence} gives the convergence of batch TD(0) to the MDP-CEE value function. Proof details are in Appendix \ref{sec:mpd_td_convergence_proof}.
\begin{restatable}[Batch Linear TD(0)
Convergence]{theorem}{thmdptdconvergence}
\label{th:mdp_td_convergence}
For any batch whose observation vectors $\{ \bx(s) | s 
\in  \widehat{\mathcal{S}}\}$ are linearly independent, there exists an 
$\epsilon > 0$  such that, for all positive $\alpha < \epsilon$ and for any 
initial
weight vector, the predictions for linear TD(0) converge under
repeated presentations of the batch with weight updates after
each complete presentation to the fixed-point (\ref{eq:mdp-dp}). 
\end{restatable}

The MDP-CEE value function highlights two sources of estimation error in the value function estimate: $P \neq \widehat{P}$ and/or $\pieval \neq \hat{\pi}$.
We describe the former as \textit{transition sampling error} and the latter as \textit{policy sampling error}.
Transition sampling error may be unavoidable in a model-free setting since we do not know $P$.
However, we do know $\pieval$ and can use this knowledge to potentially correct policy sampling error.
In the next section, we present an algorithm that uses the knowledge of $\pieval$ to correct for policy sampling error and obtain a more accurate value function estimate.

%
%
%
%

\section{Batch Linear PSEC-TD(0)}
\label{sec:ristd}

In this section, we introduce the batch policy sampling error corrected-TD(0) (PSEC-TD(0)) algorithm that corrects for the policy sampling error in batch TD learning.
%
%
From Theorem \ref{th:mdp_td_convergence}, batch TD(0) converges to the
value function for the maximum likelihood policy, $\hat{\pi},$ instead of $\pieval$.
Under this view, PSEC-TD(0) treats policy sampling error as an off-policy learning problem and uses importance sampling \citep{precup2000eligibility} to correct the weighting of TD(0) updates from $\hat{\pi}$ to $\pieval$. 
Even though importance sampling is usually associated with off-policy learning, this approach is applicable in the on- \emph{and} off-policy cases.
%

In addition to $\data$ and $\pieval$, we assume we are given a set of policies, $\Pi$.
Batch PSEC-TD(0) first computes the maximum likelihood estimate of the behavior policy:
{\small
\begin{equation*}\footnotesize
   \hat{\pi} \coloneqq \argmax_{\pi' \in \Pi} \sum_{\tau \in \data} \sum_{t=0}^{L_{\tau - 1}} \log \pi'(a_t^\tau | s_t^\tau) 
\end{equation*}
}
This estimation can be done in a number of ways.
For example, in the tabular setting
we could use the empirical count of actions in each state.
This count-based
approach is often intractable, and hence, in many problems of interest we must rely on
function approximation.
When using function approximation, the policy estimate can be obtained by minimizing a negative log-likelihood loss function. Once $\hat{\pi}$ is computed, the batch PSEC-TD algorithm is the same as Algorithm \ref{algo:batch_linear_td} with $\hat{\pi}$ replacing $\pib$ in the importance sampling ratio.
That is, for transition $(s, a, r, s')$ in $\data$, the contribution to the weight update is 
$\bm{u} \leftarrow \bm{u} + \left[\hat{\rho} \hat{y} -  \bw_i^T\bx(s)\right]\bx(s)$, 
%
%
where $\hat{\rho} \coloneqq \frac{\pieval(a | s)}{\hat{\pi}(a| s)}$ is the  PSEC weight (refer to Line 8 in 
Algorithm \ref{algo:batch_linear_td}).
Thus, PSEC makes an importance sampling correction from the \emph{empirical} to the \emph{evaluation} policy distribution.


\subsection{Convergence of Batch Linear PSEC-TD(0)}
\label{sec:psec_convergence}

Section \ref{sec:td_convergence} showed that batch TD(0) converges to two equivalent certainty-equivalence estimates.
We now define a new certainty-equivalent estimate (CEE) to which our new batch PSEC-TD(0) algorithm converges.
Intuitively, the MDP-CEE estimate (Definition \ref{def:mdp_cee}) is the exact value function for the \emph{MLE of the behavior policy, $\hat{\pi}$}, in the MLE of the MDP environment; our new algorithm converges to the exact value function for $\pieval$ in the MLE of the MDP environment, making it more data efficient than batch TD(0)  once the batch size is large enough.

We define this new CEE as the \textit{PSEC
Markov Decision Process Certainty Equivalence Estimate} (PSEC-MDP-CEE) Value
Function.

\begin{definition}{PSEC Markov Decision Process Certainty Equivalence Estimate (PSEC-MDP-CEE) Value Function.}
\label{def:psec_mdp_cee}
The PSEC-MDP-CEE is the value function, $\hat{v}_\mathtt{PSEC-MDP}^{\pi_e}$, that, $\forall s,s' \in \widehat{\sset}$, satisfies:
%
{\small
\begin{align}\label{eq:ris-mdp-dp}\footnotesize
\begin{split}
	&\hat{v}_\mathtt{PSEC-MDP}^{\pi_e}(s) = \sum_{a \in \widehat{\aset}} \pi_e(a | s) [ \bar{R}(s, a) \\&+
	\gamma \sum_{k \in \widehat{\sset}} \widehat{P}(s'|s, a) \hat{v}_{\mathtt{PSEC-MDP}}^{\pi_e}(s'))]
\end{split}
\end{align}
}
\end{definition}
Theorem \ref{th:ris_td_convergence} states that batch PSEC-TD(0) converges to the new 
PSEC-MDP-CEE value function (Equation \ref{eq:ris-mdp-dp}). Proof details are in Appendix \ref{sec:ris_td_convergence_proof}.
\begin{restatable}[Batch Linear PSEC-TD(0)
Convergence]{theorem}{thristdconvergence}
\label{th:ris_td_convergence}
For any batch whose observation vectors $\{ \bx(s) | s 
\in  \widehat{\mathcal{S}}\}$ are linearly independent, there exists an 
$\epsilon > 0$  such that, for all positive $\alpha < \epsilon$ and for any 
initial
weight vector, the predictions for linear PSEC-TD(0) converge under
repeated presentations of the batch with weight updates after
each complete presentation to the fixed-point (\ref{eq:ris-mdp-dp}).
\end{restatable}
We remark that convergence has only been shown for the on-policy setting.
While PSEC-TD(0) can be applied in the off-policy setting, it may, like other semi-gradient TD methods, diverge when off-policy updates are made with function approximation \citep{baird1995residual}.
It is possible that combining PSEC-TD(0) with Emphatic TD \citep{mahmood2015emphatic} or Gradient-TD \citep{sutton2009fast} may result in provably convergent behavior with off-policy updates, however, that study is outside the scope of this work.
%
\subsection{Extending PSEC to other TD Variants}

In general, PSEC can improve any value function learning algorithm that computes the TD-error, $\delta$, or equivalent errors.
As an example, we consider the off-policy least-squares TD (LSTD) algorithm \citep{lstd, ghiassian2018online}, which analytically computes the exact parameters that minimize the TD-error in a batch of data using the following steps:
{\small
\begin{align*}
    A &= \sum_{(s, a, s') \in \data} \left[\hat{\rho} \bx(s) (\bx(s) - \gamma \bx(s'))^T\right]\\
    \mathbf{b} &= \sum_{(s, a, s') \in \data} R(s, a, s') \bx(s) \\
    \bw &= A^{-1}\mathbf{b},
\end{align*}
}
where $\hat{\rho}$ is the PSEC weight.
Even though we primarily consider TD(0) in this work, the extension to LSTD demonstrates that PSEC-TD can be extended to other value function learning algorithms.

\section{Empirical Study}
\label{sec:exp}

In this section, we empirically study PSEC-TD to answer the following questions:
\begin{enumerate}
    \itemsep0em
    \item Does batch PSEC-TD(0) lower MSVE compared to batch TD(0)?
    \item Does batch linear PSEC-TD(0) empirically converge to its certainty-equivalence solution?
    \item Does PSEC yield benefit when applied to LSTD? 
    \item What factors does PSEC's data efficiency depend on in the function approximation setting?
\end{enumerate}

We briefly describe the RL domains used in our experiments.
\begin{itemize}
\itemsep0em
    \item \textbf{Gridworld:} In this domain, an agent
    navigates a $4 \times 4$ grid to
    reach a corner. The \emph{state
    and action spaces are discrete} and we use a tabular representation for $\hat{v}^{\pi_e}$. PSEC-TD(0) uses
    count-based estimation for $\hat{\pi}$.
    The ground truth value function is computed with dynamic programming and  the MSVE computation uniformly weights the error in each state. In Section \ref{sec:exp_tabular}, we consider a \emph{deterministic}
    gridworld, where there is no transition dynamics 
    sampling error.
    \item \textbf{CartPole:} In this domain, an agent
    controls a cart to balance a pole upright. The \emph{state space is continuous
    and action space is discrete}.  We
    only consider the on-policy setting. The evaluation policy is a neural network trained using REINFORCE \citep{REINFORCE}.
    It has $2$ hidden layers with $16$ neurons. We evaluate PSEC with varying linear and neural network representations 
    for the value function. $\hat{\pi}$ maps the raw state features to a softmax distribution over the actions with 
    varying linear and neural network architectures. Since the true value function is unknown, we follow  \citet{DBLP:journals/corr/PanWW16} and use Monte Carlo rollouts from a fixed number of states sampled from episodes following the evaluation policy to approximate the ground-truth state-values of those states. 
    We then compute the MSVE between the learned values and the average Monte Carlo return from these sampled states.
    \item \textbf{InvertedPendulum:} This domain is similar to CartPole, and the objective is the same -- to balance a pole upright. However, the \emph{state and action spaces are both continuous}. We
    only consider the on-policy setting. The evaluation policy is a neural network trained by PPO \citep{ppo}. The network has $2$ hidden layers with $64$ neurons each. We evaluate PSEC with varying linear and neural network representations 
    for the value function. The $\hat{\pi}$ estimate consists of two components: 1) a linear or neural network mapping from raw state features to the mean vector of a Gaussian distribution, and 2) parameters representing the log standard deviation of each element of the output vector. As in CartPole, we compute Monte Carlo rollouts for sampled states.    
\end{itemize}

In all experiments, the value function learning algorithm iterates over the whole batch of data until convergence, after
which the MSVE of the final value function is computed. Some experiments include a parameter sweep
over the hyperparameters, which can be found in Appendix \ref{app:exp_details}.

\subsection{Tabular Setting}
\label{sec:exp_tabular}
In this set of experiments, we consider two variants of PSEC-TD that differ in the placement of the PSEC weight: 
\begin{itemize}
    \itemsep0em
    \item PSEC-TD-Estimate: Multiplies $\hat{\rho}$ by the new estimate: $\hat{y} = R + \gamma \bw^T \bx(s').$
    \item PSEC-TD: Multiplies $\hat{\rho}$ by the TD error: $\delta = ( R + \gamma \bw^T \bx(s')) - \bw^T \bx(s).$
\end{itemize}
For off-policy TD(0), these placements of $\hat{\rho}$ are equivalent in expectation although the method using the TD-error has been reported to perform better in practice \citep{ghiassian2018online}. In
this section, we focus on the on-policy results. Appendix \ref{app:gridworld_add_results} includes
off-policy results.

\subsubsection{Data Efficiency}
Figure \ref{fig:gridworld_ristd_compare} answers our first and third empirical questions, and shows that PSEC lowers MSVE compared to batch TD(0), and a variant of TD(0), LSTD(0).
The gap between PSEC and its TD counterpart increases dramatically with more data; we discuss this observation in 
Section \ref{sec:gridworld_conv_cee}.
\begin{figure}[H]
    \centering
        \subfigure[On-policy (PSEC-)TD(0)]{\label{fig:on_grid}\includegraphics[scale=0.185]{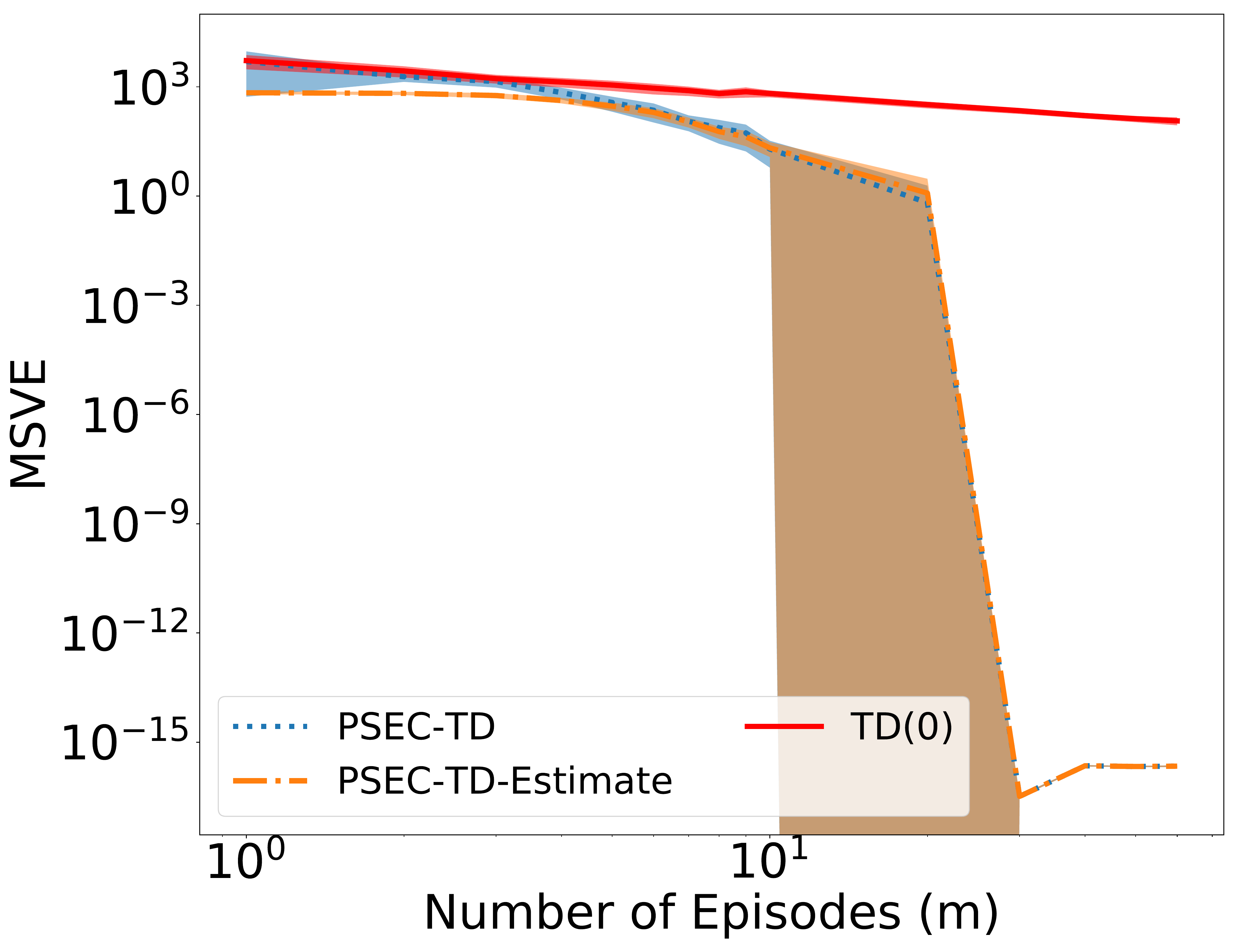}}
        \subfigure[On-Policy (PSEC-)LSTD(0)]{\label{fig:on_lstd}\includegraphics[scale=0.185]{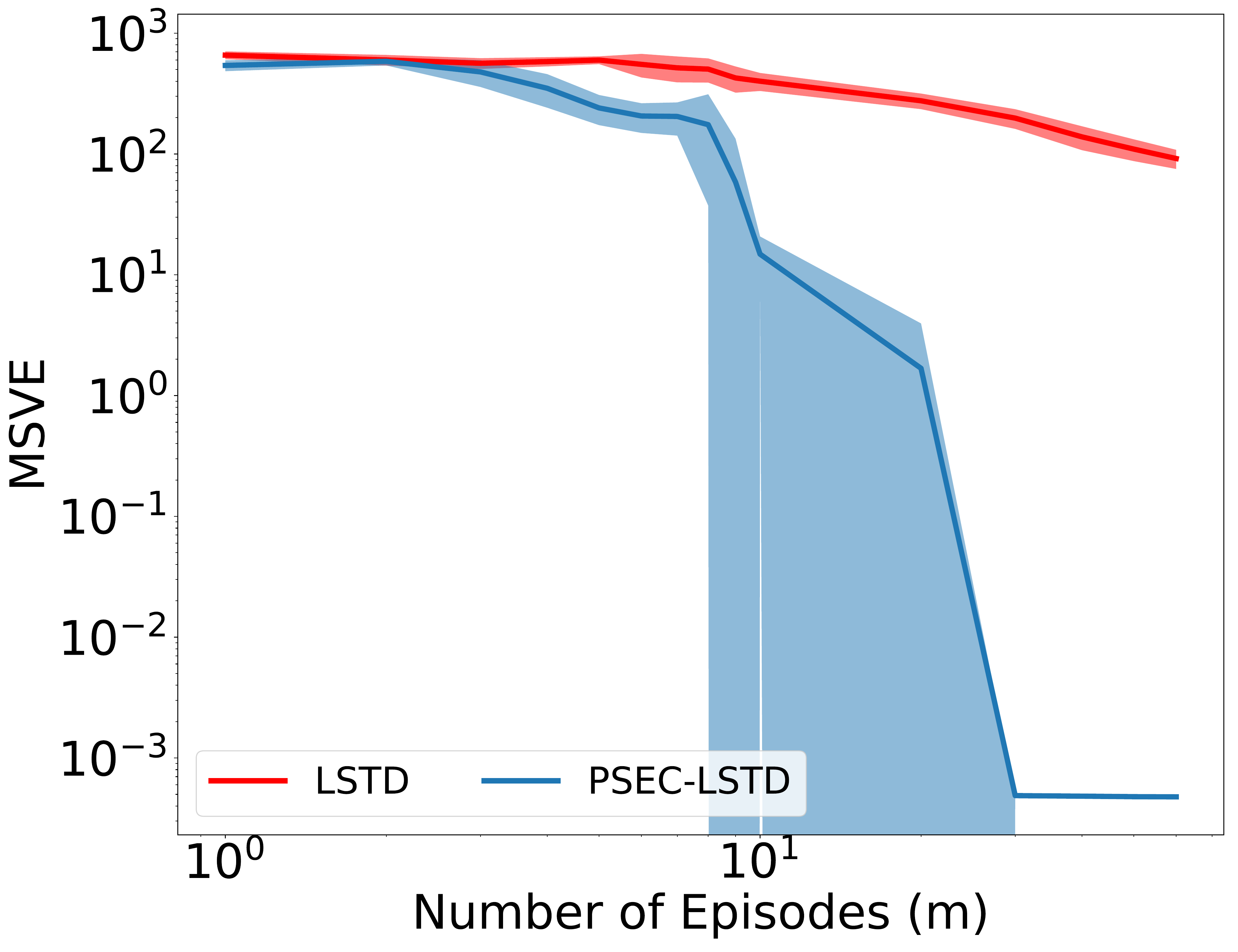}}
    \caption{\footnotesize Deterministic Gridworld experiments.
    Both axes are log-scaled. Errors are computed over $200$ trials with $95\%$ confidence intervals. Asymmetric
    confidence intervals are due to log-scaling.
    Figure \ref{fig:on_grid} and Figure \ref{fig:on_lstd} compare the data efficiency of PSEC-TD(0) and PSEC-LSTD(0) with
    their respective TD equivalents. Lower MSVE is better.}
    \label{fig:gridworld_ristd_compare}
\end{figure}

\subsubsection{Convergence to the PSEC-Certainty-Equivalence}
\label{sec:gridworld_conv_cee}

To address our second empirical question, we empirically verify that both variants of batch
linear PSEC, PSEC-TD
and PSEC-TD-Estimate, converge to the dynamic programming computed  PSEC-MDP-CEE value function
(\ref{eq:ris-mdp-dp}) in Gridworld. 
According
to Theorem \ref{th:ris_td_convergence},
batch linear PSEC-TD-Estimate converges to
the fixed-point (\ref{eq:ris-mdp-dp}) for all batch
sizes. 
\begin{figure}[H]
    \centering
        \subfigure[]{\label{fig:ristd_mdp_cee}\includegraphics[scale=0.185]{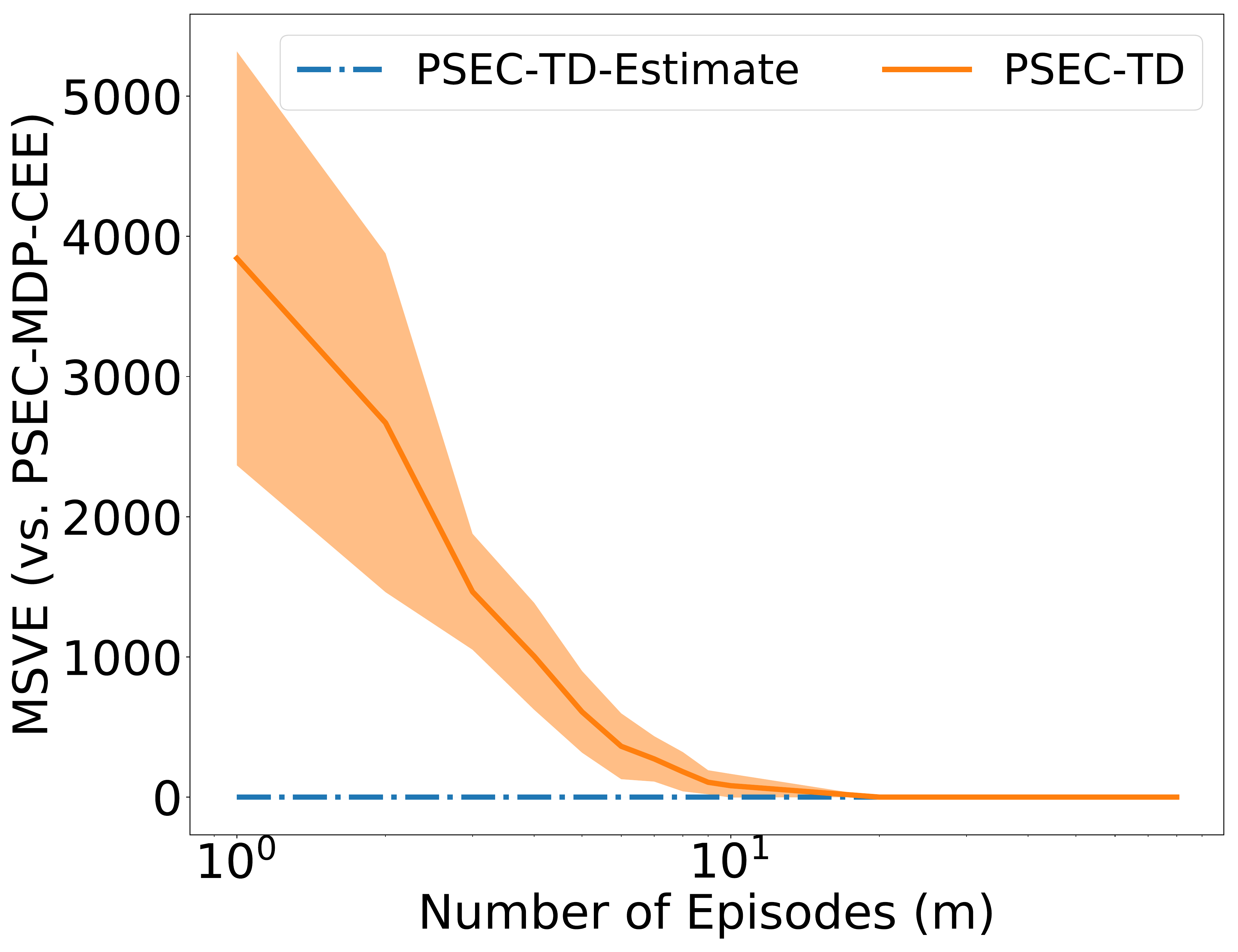}}
        \subfigure[]{\label{fig:ris_td_unvisited_ratio}\includegraphics[scale=0.185]{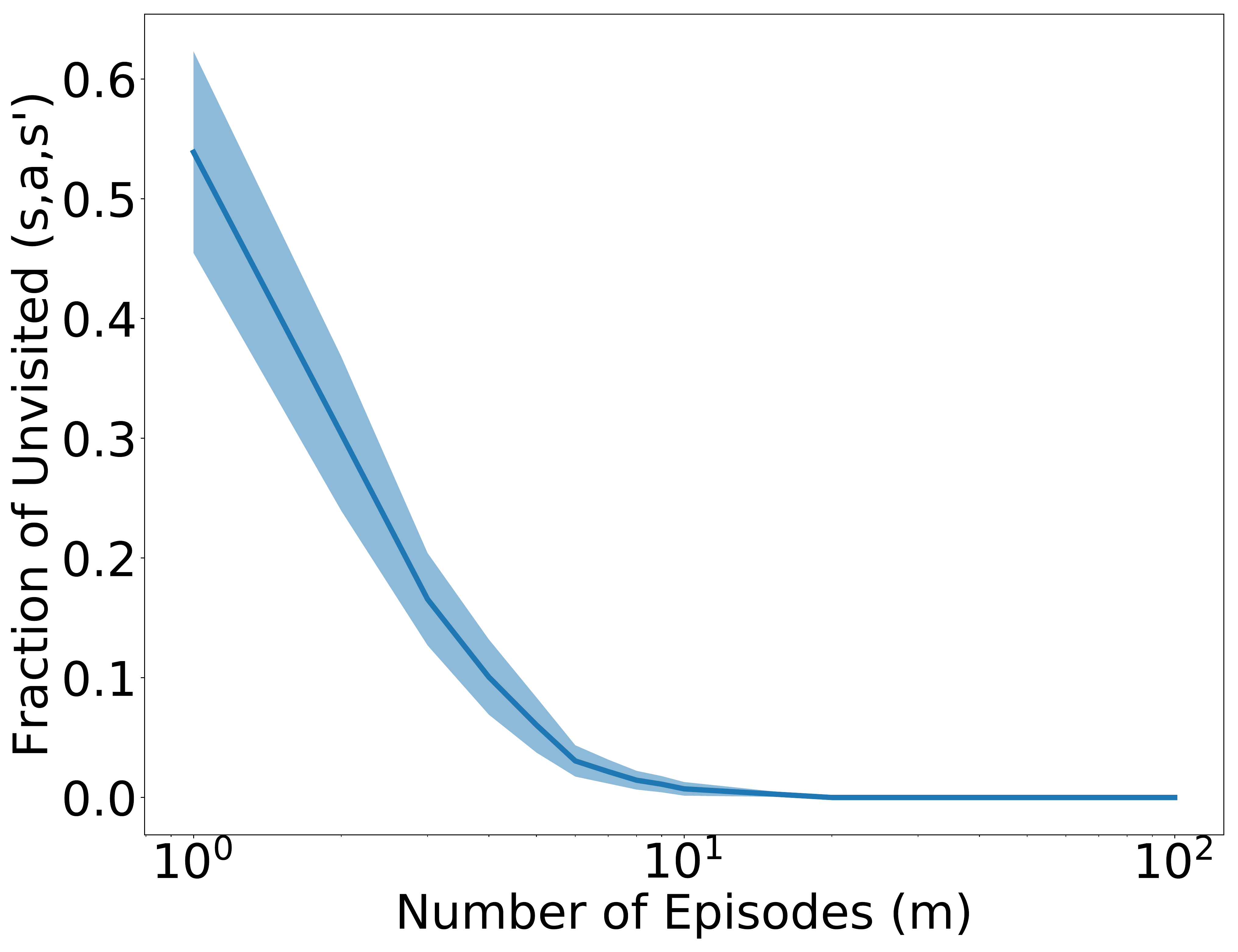}}
    \caption{\footnotesize Additional Gridworld experiments. Errors  are computed over $50$ trials with $95\%$ confidence intervals. Figure \ref{fig:ristd_mdp_cee} 
    shows  MSVE achieved by 
    variants of linear batch PSEC-TD(0), PSEC-TD and PSEC-TD-Estimate,
    with respect to the PSEC-MDP-CEE  (\ref{eq:ris-mdp-dp}).
    Figure \ref{fig:ris_td_unvisited_ratio} shows the fraction of unvisited $(s,a,s')$ tuples.}
    \label{fig:gridworld_ablation}
\end{figure}
We also empirically confirm
that the other variant of PSEC, PSEC-TD converges to the same fixed-point (\ref{eq:ris-mdp-dp}) when the following condition holds true:
only when all non-zero
probability actions for each state
in the batch have been
sampled at least once. We note that when this condition is false,
PSEC-TD-Estimate treats
the value of taking that action as $0$. For example, if a state, $s$,
appears in the batch and an action, $a$, that could take the 
agent to state $s'$ does not appear in the batch, then
PSEC-TD-Estimate treats the new estimate $R + \gamma \bw^T\bx(s')$ as $0$, which
is also done by the dynamic programming computation 
(\ref{eq:ris-mdp-dp}). We note that PSEC-TD converges to the
fixed-point (\ref{eq:ris-mdp-dp}) only when this
condition is true since the PSEC weight 
requires a fully supported probability distribution when
applied to the TD-error estimate.  From
Figure \ref{fig:ristd_mdp_cee}
and Figure \ref{fig:ris_td_unvisited_ratio}, we
can see that this condition
holds at batch size of $10$ episodes. We also note 
that PSEC-TD(0) corrects policy sampling error for
each $(s,a,s')$ transition. Thus, when all such transitions are
visited, PSEC fully corrects for all policy sampling error, which
occurs at batch size of $10$ episodes in this \emph{deterministic}
gridworld.

\subsection{Function Approximation Setting}
\label{sec:fa_cont_s_dis_a}

In this set of experiments, we answer our first and fourth empirical questions concerning function approximation in PSEC. Our experiments focus on applying only the second variant of PSEC, PSEC-TD, since we found that PSEC-TD-Estimate  diverges. The results shown below are for the on-policy case. In addition to results of PSEC as a function of data size, we conduct
 experiments on a fixed batch size to better understand how components of the PSEC training process impact
performance. Finally, we give a practical recommendation
for use of batch PSEC-TD(0).

In these
experiments, we have three function approximators: one for the value function; one to estimate the behavior policy;
and the pre-learned behavior policy itself. When any are referred to as ``fixed", it means its architecture is unchanged. 
Due to space constraints, we only show a subset of results from CartPole and Inverted Pendulum; however, a
fuller set of experiments
can be found in Appendix \ref{app:cartpole_exps} and \ref{app:inverted_pend_exps}. Note that  in all PSEC training settings, PSEC performs gradient steps using the full batch of data, uses a separate
batch of data as the validation data, and terminates training according to early stopping. Statistical significance is 
determined by Welch's test \citep{welch} with a significance level of $0.05$. For hyperparameter details
refer to Appendix \ref{app:cartpole_exps} and \ref{app:inverted_pend_exps}.

\subsubsection{Data Efficiency}
In CartPole, PSEC produced statistically significant improvement over TD 
in all batch sizes except $500$. In InvertedPendulum, like in Gridworld,
the improvement was marginal for smaller batch sizes, but produced statistically
significant improvement with larger batch sizes. As data gets larger, we observe
that both methods perform similarly for two reasons: 1) the PSEC weight approaches $1$, which effectively becomes TD(0) and 2) saturation in value function
representation capacity, which we discuss in Section
\ref{sec:func_approx_arch_model_sel}. Note that
while a thorough parameter sweep can
achieve better performance, it
is computationally expensive. The results shown here are with sweeps
over only the value function model class and PSEC learning rate.
\begin{figure}[H]
    \centering
    \subfigure[CartPole]{\label{fig:exp_cart_msve_batch}\includegraphics[scale=0.19]{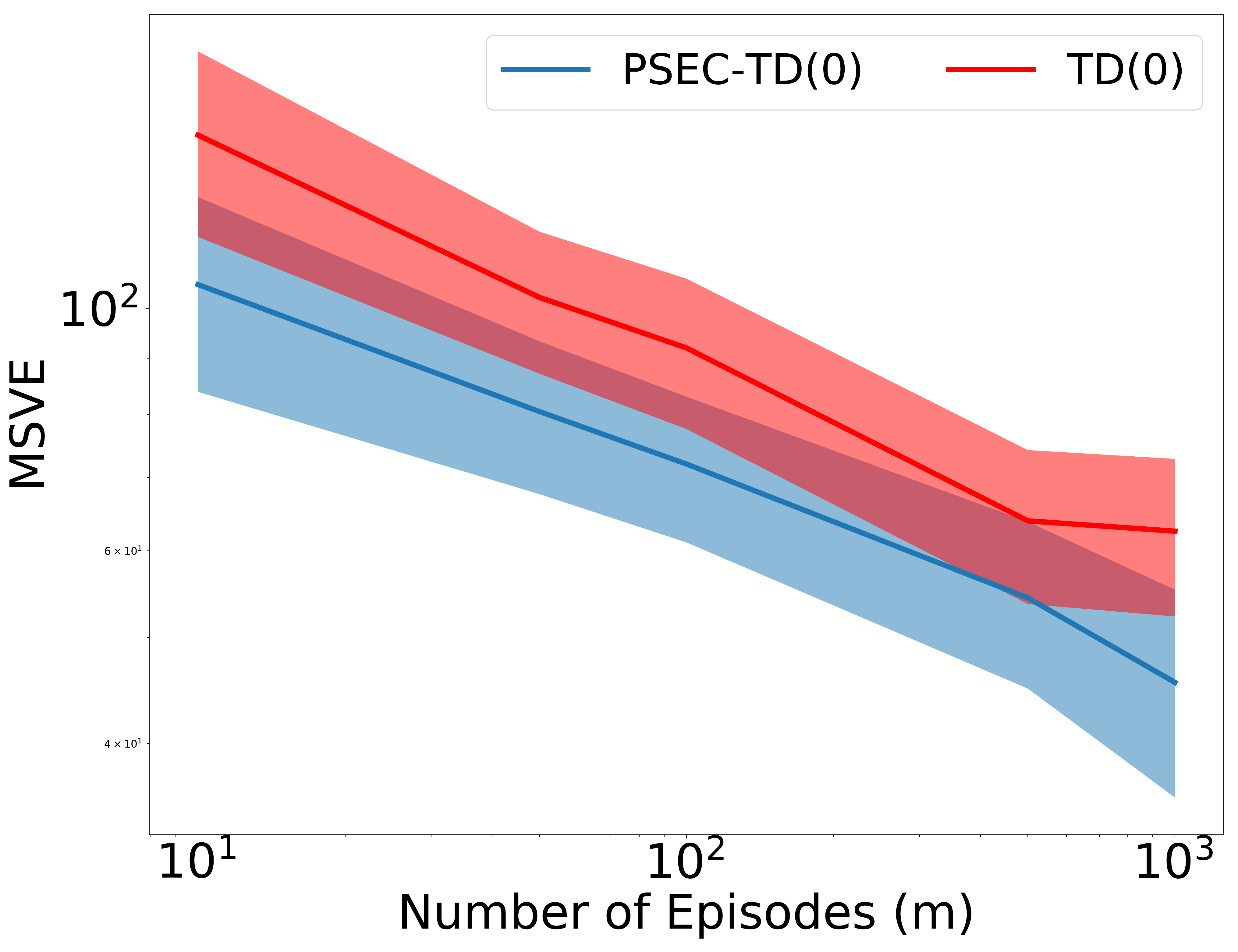}}
    \subfigure[InvertedPendulum]{\label{fig:exp_invpen_msve_batch}\includegraphics[scale=0.19]{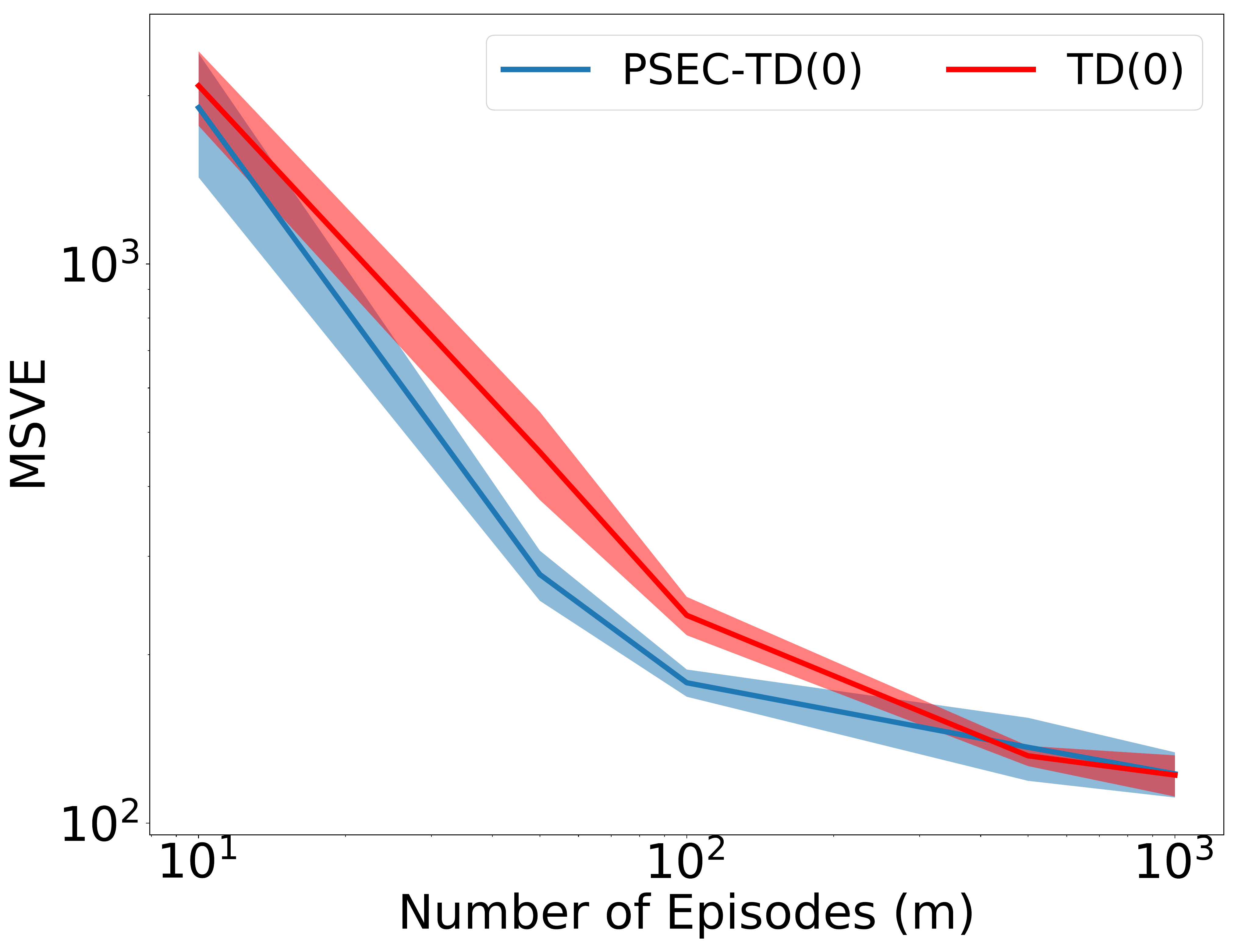}}
    \caption{\footnotesize Comparing data efficiency of PSEC and TD
    on different batch sizes. Results for Figure \ref{fig:exp_cart_msve_batch}
    and Figure \ref{fig:exp_invpen_msve_batch} are averaged over $400$  and
    $250$
    trials resp. with shaded region of $95\%$ confidence. Both axes are log-scaled. Lower MSVE is better.}
\end{figure}

\subsubsection{Architecture Model Selection}
\label{sec:func_approx_arch_model_sel}

Figure \ref{fig:exp_psec_vs_vf} illustrates the impact of different value function classes on the data efficiency of
TD and PSEC, while holding the PSEC model and behavior policy architectures fixed, on CartPole. We generally found that more expressive value function representations resulted in better data efficiency by both algorithms. We also found that the gap between 
PSEC and TD increased as the VF representation became more expressive. We hypothesize that even though
PSEC finds a more accurate fixed point than TD in the space of all value functions, the shown  
 difference between the two algorithms is dependent on the space of representable value functions --  a more representable function class can capture the difference between the two
algorithms better. The lighter shades mean that any difference between PSEC and TD was statistically insignificant.

Figure \ref{fig:exp_psec_arch} compares the data efficiency of PSEC against TD with varying PSEC neural network
model architectures, while the value function and behavior policy architectures are fixed, on CartPole.  In general, we found that more expressive network models
produced better PSEC weights since they were able to better capture the MLE of
the policy from the data. Unlike the NN PSEC policies, the linear function PSEC policy did not
produce a statistically significant improvement over TD. 

\begin{figure}[]
    \centering
    \subfigure[]{\label{fig:exp_psec_vs_vf}\includegraphics[scale=0.33]{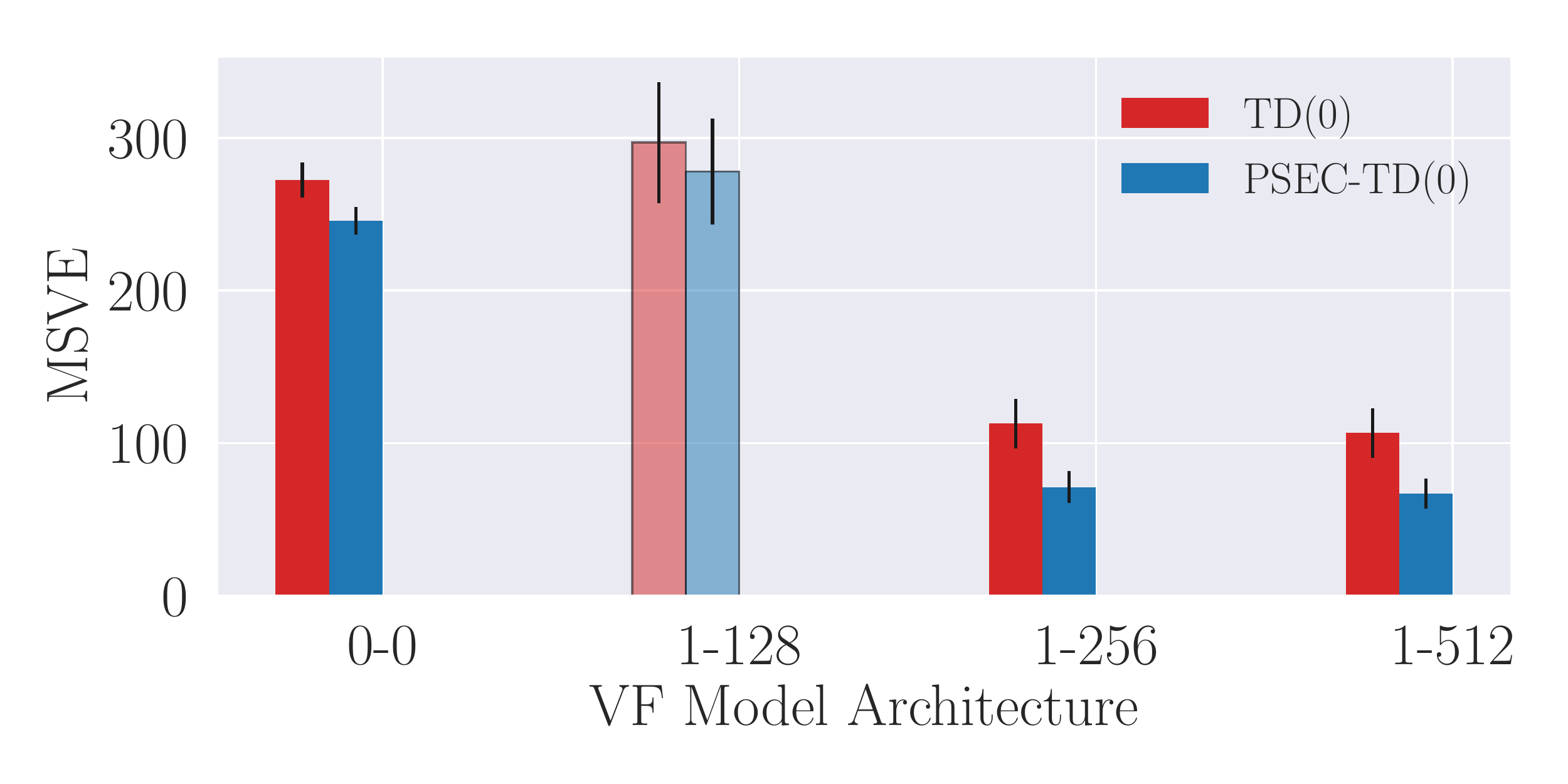}}
    \subfigure[]{\label{fig:exp_psec_arch}\includegraphics[scale=0.33]{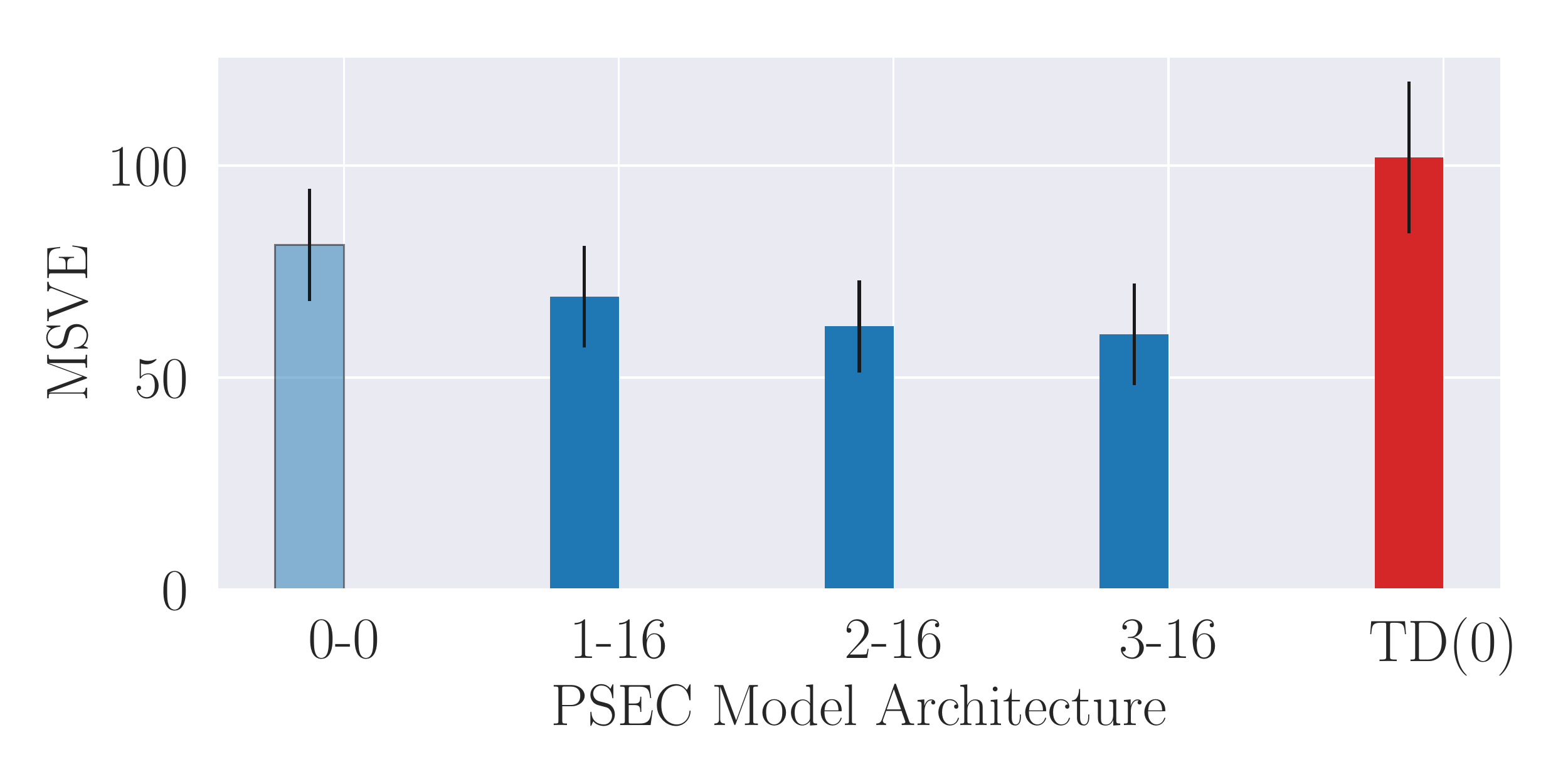}}
    \caption{\footnotesize Figure \ref{fig:exp_psec_vs_vf} and Figure \ref{fig:exp_psec_arch} compare data efficiency 
    of PSEC, with varying VF model architectures, and PSEC, with 
    varying model
    arch, respectively against TD on CartPole. Both
     use a batch size of $10$ episodes, and results shown are averaged over $300$ trials with error bars of $95\%$ confidence. Darker shades represent statistically significant results. The label on
    the $x$ axis shown is ($\#$ hidden layers - $\#$ neurons). Lower MSVE is better.}
\end{figure}



\subsubsection{Sensitivity Studies}
Due to space limitations, we defer the empirical analysis of
other effects to Appendices \ref{app:cartpole_exps} and
\ref{app:inverted_pend_exps}. Figure \ref{fig:exp_psec_vs_lr}
and Figure \ref{fig:exp_psec_vs_lr_cont_act} indicate that a
small learning rate for the PSEC model is preferred. Figure \ref{fig:exp_psec_overfitting} and \ref{fig:exp_psec_overfitting_cont_act}
indicate that some overfitting by the PSEC model is tolerable,
and perhaps, preferable, but extreme overfitting can degrade
performance.
\paragraph{Practical Recommendation} Based on our experiments, 
we recommend the following: 1) an expressive value function that can 
represent the more accurate fixed-point of PSEC-TD, 2) a PSEC model class 
that can represent the true behaviour policy but with awareness that 
extreme overfitting may hamper
performance, and 3) a small learning rate.

\section{Related Work}
\label{sec:related_work}

In this section, we discuss the literature on importance sampling with an estimated behavior policy and reducing sampling error in reinforcement learning.

The approach in this work 
has been motivated by prior work showing that importance sampling with an estimated behavior policy can lower variance when estimating an expected value in RL.
\citet{ICML2019-Hanna} introduce a family of methods called regression importance sampling methods (RIS) and show that they have lower variance than importance sampling with the true behavior policy.
\citet{hanna2019reducing} show that a similar technique led to more sample-efficient policy gradient learning.
These works are related to work in the multi-armed bandit  
 \cite{Li2015TowardMO, DBLP:journals/corr/abs-1809-03084, DBLP:journals/corr/abs-1808-00232}, causal inference \cite{doi:10.1111/1468-0262.00442,10.2307/2289440}, and
Monte Carlo integration \cite{RePEc:oup:biomet:v:94:y:2007:i:4:p:985-991, Delyon_2016} literature.
In contrast, our work focuses on \textit{value function learning}, where the focus is on learning the expected return at every state visited by the agent instead of across a set of actions (multi-armed bandit) or for some start states that are a subset of all the states the agent visits.

PSEC-TD(0) corrects policy sampling error 
through importance sampling with an estimated behavior policy.
Other works avoid policy sampling error entirely by computing analytic expectations.
Expected SARSA \cite{4927542}, learns action-values by
analytically computing the expected return of the next state during 
bootstrapping as opposed to using
the value of the sampled next action.
The Tree-backup algorithm \cite{Precup2000EligibilityTF} extends Expected SARSA to a multi-step algorithm.
$Q(\sigma)$ \cite{asis2017multistep} unifies SARSA \cite{NIPS1995_1109, Rummery94on-lineq-learning}, Expected SARSA, and Tree-backups, to
find a balance between sampling and analytic expectation computation. 
Our work is distinct from these in that
we focus on learning state values which may be preferable for prediction as well as a variety of actor-critic approaches \cite{NIPS1999_1786, mnih2016asynchronous}.
To the best of our knowledge, no other approach exists for correcting policy sampling error when learning state values. 

\section{Summary and Discussion}
\label{sec:disc}

In batch value function approximation, we observed that TD(0) may converge to an inaccurate estimate of the value function due to policy sampling error.
We proposed batch PSEC-TD(0) as a method to correct this error and showed that it leads to a more data efficient
estimator  than batch TD(0). In this paper, we theoretically analyzed PSEC-TD and empirically evaluated it in the tabular and function approximation
settings. Our empirical study validated that PSEC converges to a more accurate fixed point than TD, and 
studied how the numerous components in the PSEC training setup impact its data efficiency with respect to TD.

Despite the data efficiency benefits that batch  PSEC-TD(0)
introduced, there are limitations.
First, it requires knowledge of the evaluation policy, which on-policy TD(0) does not.
%
%
This comparative disadvantage is only for the on-policy setting as both TD(0) and PSEC-TD(0) require knowledge of the evaluation policy for the off-policy setting.
Additionally, PSEC-TD(0), in the off-policy case, has the advantage of not requiring knowledge of the behavior policy $\pib$.
Second, the policy estimation step required by PSEC-TD(0) could potentially be computationally expensive.
%
For instance, requiring the computation and storage of $\mathcal{O}(|\sset||\aset|)$ parameters in the tabular setting.
%
%


There are several directions for future work.
First, our work
focused on batch
 TD(0).
We expect that a variant of PSEC can improve value function learning with $n$-step TD and TD($\lambda$). 
Second, with an improved value function learning
algorithm, it would be interesting to see if an agent can
learn better control policies.  Third, it would be interesting
to theoretically and empirically study PSEC when learning
the state-action values. Finally, automatically finding the optimal training setting
for PSEC in the function approximation setting is another important direction for future work.

\section*{Acknowledgements}
We thank Darshan Thaker, Elad Liebman, Reuth Mirsky, Sanmit Narvekar, Scott
Niekum, Sid Desai,  Yunshu Du, and the anonymous reviewers
for reviewing our work and for their helpful
comments. This work has taken place in the Learning Agents Research
Group (LARG) at the Artificial Intelligence Laboratory, The University
of Texas at Austin.  LARG research is supported in part by grants from
the National Science Foundation (CPS-1739964, IIS-1724157,
NRI-1925082), the Office of Naval Research (N00014-18-2243), Future of
Life Institute (RFP2-000), Army Research Office (W911NF-19-2-0333),
DARPA, Lockheed Martin, General Motors, and Bosch.  The views and
conclusions contained in this document are those of the authors alone.
Peter Stone serves as the Executive Director of Sony AI America and
receives financial compensation for this work.  The terms of this
arrangement have been reviewed and approved by the University of Texas
at Austin in accordance with its policy on objectivity in research.

\medskip

\bibliography{main}
\bibliographystyle{icml2020}

\onecolumn
\newpage
\normalsize
\appendix

\section{Matrix Notation for Proofs}
\label{notation:matrix}

In this section, we introduce matrix-related notation for the proofs. Part of this notation is derived from
 \citet{sutton1988learning}.
We refer to state features using vectors indexed by the state. So features $\bx(i)$ for state $i$ is referred to as $\bx_i$.
Reward $r(j|i, a)$ (the reward for transitioning to state $j$ from state $i$ after taking action $a$) is referred to by $r_{ij}^a$.
The policy $\pi(a|i)$ is $\pi_i^a$ and the transition function $P(j|i, a)$ is accordingly the value $p_{ij}^a$.
$\mathcal{N}$ and $\mathcal{T}$ are the set of non-terminal and terminal states respectively, and $\mathcal{A}$ is the set of possible actions.
$\mathcal{D}$ is the batch of data used to train the value function. For any transition from a non-terminal state
to a (non-) terminal state, $i \to j$, $1 \leq i \leq |\mathcal{N}|$ and $1 \leq j \leq |\mathcal{N} \cup \mathcal{T}|$.
Finally, the maximum-likelihood estimate (MLE) of the above quantities according to $\mathcal{D}$, is given with a hat ( $\hat{}$ ) on top of the quantity. Further notations that are used in the proofs are explained when introduced.

\begin{center}
\begin{tabular}{ m{4em}|m{31em}|m{5em} } 
 \toprule 
 Notation & \center{Description} & Dimension \\ \midrule
 $\sset$ & set of states & $|\sset|$ \\
 \hline
  $\widehat{\sset}$ & set of states that appear in batch $\mathcal{D}$ & $|\widehat{\sset}|$ \\
 \hline
 $\aset$ & set of actions & $|\aset|$ \\
 \hline
   $\widehat{\aset}_i$ & set of actions that appear in batch $\mathcal{D}$ when agent is in state $i$ & $|\widehat{\aset}_i|$ \\
 \hline
 $\mathcal{N} \subset \sset$ & non-terminal states & $|\mathcal{N}|$\\
 \hline
   $\widehat{\mathcal{N}} \subset \widehat{\sset}$ & non-terminal states that appear in batch $\mathcal{D}$ & $|\widehat{\mathcal{N}}|$ \\
 \hline 
 $\mathcal{T} \subset \sset$ & terminal states & $|\mathcal{T}|$ \\
 \hline
   $\widehat{\mathcal{T}} \subset \widehat{\sset}$ & terminal states that appear in batch $\mathcal{D}$ & $|\widehat{\mathcal{T}}|$ \\
 \hline 
 $I$ & identity matrix &  $|\mathcal{N}|\times|\mathcal{N}|$ \\
 \hline
 $p_{jk}^{\pi}$ & probability of transitioning from state $j$ to state $k$ under a policy, $\pi$ & - \\
  \hline
 $p^a_{jk}$ & probability of transitioning from state $j$ to state $k$ after taking action $a$ & - \\
 \hline
 $\bar{r}_j$ & mean reward when transitioning from state $j$ & -\\
 \hline
  $\bar{r}_{ij}^a$ & mean reward when transitioning from state $i$ to state $j$ after taking action $a$ & -\\
 \hline
 $Q$ & $Q_{jk}\coloneqq p_{jk}^{\pi}$ & $|\mathcal{N}|\times|\mathcal{N}|$ \\
 \hline
 $[\mathbf{m}]_i$ & expected reward on transitioning from state $i$ to non-terminal state $j$ i.e. 
  $\sum_{j \in \mathcal{N}} p_{ij}^{\pi} r_{ij}$ or  $\sum_{j \in \mathcal{N}} \sum_{a \in \mathcal{A}} \pi_i^a p_{ij}^ar_{ij}^a$
    & $|\mathcal{N}|$\\
 \hline
 $[\mathbf{h}]_i$ & expected reward on transitioning from state $i$ to non-terminal state $j$  to terminal state i.e. $\sum_{j \in \mathcal{T}} p_{ij}^{\pi} r_{ij} $
 or $\sum_{j \in \mathcal{T}} \sum_{a \in \mathcal{A}} \pi_i^a p_{ij}^ar_{ij}^a$& $|\mathcal{N}|$ \\
 \hline
 $d_{\pieval}(i)$ & $\forall i \in \sset$ weighted proportion of time spent in state $i$ under policy $\pieval$ & -\\
\bottomrule
\end{tabular}
\end{center}

\section{Fixed-point for an MDP in the Per-step Reward and Discounted Case}
\label{appendix:fixedpoints}

In this section, we establish several fixed-points that we expect the value function to converge to in the
discounted per-step reward case for an MRP and MDP. Note that \citet{sutton1988learning} derived these fixed-points for MRPs when rewards are only received on termination and there is no discounting.
We first specify the fixed-points for an MRP in the discounted per-step reward case, and then extend this result for MDPs. 

For both an MRP and MDP, we establish two types of fixed-points in the per-step reward and discounted case. The first fixed-point is the \textit{true} 
fixed-point in that it is the value function computed assuming that 
we have access to the true policy and transition dynamics distributions. Ideally, we would like
our value function learning algorithms to converge to this fixed-point. The second fixed-point is the 
\textit{certainty-equivalence estimate} fixed-point, which is the value function computed using the maximum-likelihood
estimates of the policy and transition dynamics from a batch of fixed data. We note that due to sampling error
in the policy and transition dynamics, the certainty-equivalence estimate is an \textit{inaccurate} estimate of
the true value function. Finally, and only for MDPs, we specify the fixed-point that we expect the value function to
learn after applying PSEC.


\subsection{MRP True Fixed-Point}
The true value function $v$ for a state, $i \in \mathcal{N}$, induced by the policy-integrated transition dynamics $p^{\pi}$, reward function $r$, and policy, $\pi$, is given by:
\begin{align*}
     v(i) &= \sum_{j \in \mathcal{N} \cup \mathcal{T}} p_{ij}^{\pi} \left[r_{ij} +
     \gamma v(j)\right] &\text{Bellman equation} \\
     &= \sum_{j \in \mathcal{N}} p_{ij}^{\pi} \left[r_{ij} +
     \gamma v(j)\right] + \sum_{j \in \mathcal{T}} p_{ij}^{\pi}r_{ij} &\text{expected
     return from $\mathcal{T}, v(\mathcal{T}) = 0$}\\
     &= \sum_{j \in \mathcal{T}} p_{ij}^{\pi} r_{ij} +  \sum_{j \in \mathcal{N}} p_{ij}^{\pi}\left[r_{ij} +
     \gamma\left[\sum_{k \in \mathcal{N} \cup \mathcal{T}} p_{jk}^{\pi}\left[r_{jk} +
     \gamma v(k)\right]\right]\right] & \text{recursively apply $v(i)$}
\end{align*}
\begin{align*}
     v(i) &= \sum_{j \in \mathcal{T}} p_{ij}^{\pi} r_{ij} + \sum_{j \in \mathcal{N}} p_{ij}^{\pi} r_{ij} +
     \gamma \sum_{j \in \mathcal{N}} p_{ij}^{\pi} \sum_{k \in \mathcal{N} \cup \mathcal{T}}  p_{jk}^{\pi}r_{jk}  \\&+
     \gamma^2 \sum_{j \in \mathcal{N}} p_{ij}^{\pi} \sum_{k \in \mathcal{N} \cup \mathcal{T}}p_{jk}^{\pi} v(k)\\
     &= \sum_{j \in \mathcal{T}} p_{ij}^{\pi} r_{ij} + \sum_{j \in \mathcal{N}} p_{ij}^{\pi} r_{ij} \\ &+
     \gamma \sum_{j \in \mathcal{N}} p_{ij}^{\pi} \sum_{k \in \mathcal{N} } p_{jk}^{\pi}r_{jk} + 
     \gamma \sum_{j \in \mathcal{N}} p_{ij}^{\pi} \sum_{k \in \mathcal{T}} p_{jk}^{\pi}r_{jk} \\&+
     \gamma^2 \sum_{j \in \mathcal{N}} p_{ij}^{\pi} \sum_{k \in \mathcal{N}} p_{jk}^{\pi} v(k) 
     &\text{splitting $\mathcal{N}$ and $\mathcal{T}$}\\
\end{align*}

 We define vectors, $\mathbf{h}$
and $\mathbf{m}$ with,
$[\mathbf{h}]_{i} = \sum_{j \in \mathcal{T}} p_{ij}^{\pi} r_{ij}$,  $[\mathbf{m}]_{i} = \sum_{j \in \mathcal{N}} p_{ij}^{\pi} r_{ij}$,
and $Q$ is the true transition matrix of the Markov reward process induced by $\pi$ and $P$, i.e., $[Q]_{ij} = p_{ij}^{\pi}$. Then continuing from
above, we have
\begin{align}
     v(i) &= [\mathbf{h}]_i + [\mathbf{m}]_i + \gamma Q[\mathbf{h}]_i + \gamma Q[\mathbf{m}]_i + \gamma^2 Q^2[\mathbf{h}]_i + \gamma^2 Q^2[\mathbf{m}]_i
     + \ldots &\text{unrolling $v(\mathcal{N} \cup \mathcal{T}$)}\\
     &= \left[\sum\limits_{k = 0}^{\infty}(\gamma Q)^k(\mathbf{m} + \mathbf{h})\right]_i\\
    v(i)  &= \left[(I - \gamma Q)^{-1} (\mathbf{m} + \mathbf{h})\right]_i \label{eq:mrp-truefixedpoint}
\end{align}

The existence of the limit and inverse are assured
by Theorem A.1 in \citet{sutton1988learning}. The theorem
is applicable here since $\lim_{k\to\infty} (\gamma Q)^k = 0$.

\subsection{MRP Certainty-Equivalence Fixed -Point}
\label{appendix:mrp_cee_fixed_point}

For the certainty-equivalence fixed-point, we consider a batch of data, $\mathcal{D}$. We follow the same steps and similar notation from Equation (\ref{eq:mrp-truefixedpoint}),
with the slight modification that the maximum-likelihood estimate (MLE) of the above quantities according to
$\mathcal{D}$, is given with a hat ( $\hat{}$ ) on top of the quantity.
The observed sets of non-terminal and terminal states in the batch are given by  $\widehat{\mathcal{N}}$ and $\widehat{\mathcal{T}}$ respectively.

Then similar to above, we can derive the certainty-equivalence estimate of the value function according to
the MLE of the MRP transition dynamics from the batch for a particular state $i$, $\forall i \in \widehat{N}$ is:

\begin{align}
    \hat{v}(i)  &= \left[(I - \gamma \widehat{Q})^{-1} (\hat{\mathbf{m}} + \hat{\mathbf{h}})\right]_i \label{eq:mrp-ceefixedpoint}
\end{align}

\subsection{MDP True Fixed-Point}

The true value function, $v^\pi$, for a policy, $\pi$, for a state $i$, $\forall i \in  \mathcal{N}$, induced by the transition dynamics
and reward function, $p$ and $r$ is given by:

\begin{align*}
     v^\pi(i) &= \sum_{a \in \mathcal{A}} \pi_i^a\sum_{j \in \mathcal{N} \cup \mathcal{T}} p_{ij}^a \left[r_{ij}^a +
     \gamma v^\pi(j)\right] &\text{Bellman equation}\\
     &= \sum_{a \in \mathcal{A}} \pi_i^a\sum_{j \in \mathcal{N}} p_{ij}^a \left[r_{ij}^a +
     \gamma v^\pi(j)\right] + \sum_{a \in \mathcal{A}} \pi_i^a\sum_{j \in \mathcal{T}} p_{ij}^ar_{ij}^a &\text{expected
     return from $\mathcal{T}, v^\pi(\mathcal{T}) = 0$}
\end{align*}
\begin{align*}
     v^\pi(i) &= \sum_{a \in \mathcal{A}} \pi_i^a \sum_{j \in \mathcal{T}} p_{ij}^ar_{ij}^a + \sum_{a \in \mathcal{A}}  \pi_i^a\sum_{j \in \mathcal{N}} p_{ij}^a \left[r_{ij}^a +
     \gamma\left[\sum_{a' \in \mathcal{A}}  \pi_j^{a'}\sum_{k \in \mathcal{N} \cup \mathcal{T}} p_{jk}^{a'}\left[r_{jk}^{a'} +
     \gamma v^\pi(k)\right]\right]\right]\\
     &= \sum_{j \in \mathcal{T}} \sum_{a \in \mathcal{A}} \pi_i^a p_{ij}^ar_{ij}^a + \sum_{j \in \mathcal{N}} \sum_{a \in \mathcal{A}} \pi_i^ap_{ij}^a r_{ij}^a +
     \gamma \sum_{j \in \mathcal{N}} \sum_{a \in \mathcal{A}} \pi_i^a p_{ij}^a \sum_{k \in \mathcal{N} \cup \mathcal{T}} \sum_{a' \in \mathcal{A}} \pi_j^{a'}p_{jk}^{a'}r_{jk}^{a'}  \\&+
     \gamma^2 \sum_{j \in \mathcal{N}} \sum_{a \in \mathcal{A}} \pi_i^a p_{ij}^a \sum_{k \in \mathcal{N} \cup \mathcal{T}}\sum_{a' \in \mathcal{A}}\pi_j^{a'}p_{jk}^{a'} v^\pi(k)\\
     &= \sum_{j \in \mathcal{T}} \sum_{a \in \mathcal{A}} \pi_i^a p_{ij}^ar_{ij}^a + \sum_{j \in \mathcal{N}} \sum_{a \in \mathcal{A}} \pi_i^ap_{ij}^a r_{ij}^a \\ &+
     \gamma \sum_{j \in \mathcal{N}} \sum_{a \in \mathcal{A}} \pi_i^a p_{ij}^a \sum_{k \in \mathcal{N} } \sum_{a' \in \mathcal{A}} \pi_j^{a'}p_{jk}^{a'}r_{jk}^{a'} + 
     \gamma \sum_{j \in \mathcal{N}} \sum_{a \in \mathcal{A}} \pi_i^a p_{ij}^a \sum_{k \in \mathcal{T}} \sum_{a' \in \mathcal{A}} \pi_j^{a'}p_{jk}^{a'}r_{jk}^{a'} \\&+
     \gamma^2 \sum_{j \in \mathcal{N}} \sum_{a \in \mathcal{A}} \pi_i^a p_{ij}^a \sum_{k \in \mathcal{N}}\sum_{a' \in \mathcal{A}} \pi_j^{a'}p_{jk}^{a'} v^\pi(k)
     &\text{splitting $\mathcal{N}$ and $\mathcal{T}$}\\
\end{align*}

Similar to earlier, we have vectors, $\mathbf{h}$
and $\mathbf{m}$ with,
$[\mathbf{h}]_{i} = \sum_{j \in \mathcal{T}} \sum_{a \in \mathcal{A}} \pi_i^a p_{ij}^ar_{ij}^a$,  $[\mathbf{m}]_{i} = \sum_{j \in \mathcal{N}} \sum_{a \in \mathcal{A}} \pi_i^a p_{ij}^ar_{ij}^a$,
and $Q$ is the true transition matrix of the Markov reward process induced by $\pi$ and $P$, i.e., $[Q]_{ij} = \sum_a \pi_i^ap_{ij}^a$. The terms are not overloaded since the expectation over the true policy yields the same values. Then continuing from
above, we have
\begin{align}
     v^\pi(i) &= [\mathbf{h}]_i + [\mathbf{m}]_i + \gamma Q[\mathbf{h}]_i + \gamma Q[\mathbf{m}]_i + \gamma^2 Q^2[\mathbf{h}]_i + \gamma^2 Q^2[\mathbf{m}]_i
     + \ldots &\text{unrolling $v^\pi(\mathcal{N} \cup \mathcal{T}$)}\\
     &= \left[\sum\limits_{k = 0}^{\infty}(\gamma Q)^k(\mathbf{m} + \mathbf{h})\right]_i\\
    v^\pi(i)  &= \left[(I - \gamma Q)^{-1} (\mathbf{m} + \mathbf{h})\right]_i \label{eq:truefixedpoint}
\end{align}

The existence of the limit and inverse are assured
by Theorem A.1 in \citet{sutton1988learning}. The theorem
is applicable here since $\lim_{k\to\infty} (\gamma Q)^k = 0$. 

\subsection{MDP Certainty-Equivalence Fixed-Point}
\label{appendix:mdp_cee_fixed_point}

Similar to the above subsection, for certainty-equivalence fixed-point, we consider a batch of data, $\mathcal{D}$, with the maximum-likelihood estimate (MLE) of the above quantities according to
$\mathcal{D}$ given with a hat ( $\hat{}$ ) on top of the quantity.
The observed sets of non-terminal and terminal states in the batch are given by  $\widehat{\mathcal{N}}$ and $\widehat{\mathcal{T}}$ respectively.

Then similar to above, we can derive the certainty-equivalence estimate of the value function according to
the MLE of the policy and transition dynamics from the batch for a particular state $i$, $\forall i \in \widehat{N}$ is:

\begin{equation}
    v^{\hat{\pi}}(i)  = \left[(I - \gamma \widehat{Q})^{-1} (\mathbf{\hat{m}} + \mathbf{\hat{h}})\right]_i \label{eq:ceefixedpoint}
\end{equation}
 
 This fixed-point is called the certainty-equivalence estimate (CEE) \citep{sutton1988learning} for an MDP. We note that 
 MLE of the policy and transition dynamics according to the batch may not be representative
 of the true policy and transition dynamics. In that case, MDP-CEE (Equation (\ref{eq:ceefixedpoint})) is inaccurate
 with respect to Equation (\ref{eq:truefixedpoint}) due to policy and transition dynamics \textit{sampling error}.
 
 \subsection{Policy Sampling Error Corrected MDP Certainty-Equivalence Fixed-Point}
\label{appendix:psec_mdp_cee_fixed_point}

We now derive a new fixed-point, the \textit{policy sampling error corrected MDP certainty-equivalence fixed-point}. This 
fixed-point corrects the policy sampling error that occurs in the value function given by Equation (\ref{eq:ceefixedpoint}),
making the estimation more accurate with respect to the true value function given by Equation (\ref{eq:truefixedpoint}).

We introduce the PSEC weight, $\hat{\rho}_i^a = \frac{\pi_i^a}{\hat{\pi}_i^a}$, 
with $\pi$ being the policy that we are interested in evaluating and $\hat{\pi}$ being the MLE of the policy
according to batch $\mathcal{D}$. $\hat{\rho}$ is then applied to the above quantities to introduce a slightly
modified notation. In particular, $\hat{\rho}$ applied to $\widehat{Q}$ results in $[\widehat{U}]_{ij} =\sum_{a \in \mathcal{\hat{A}}_i} \hat{\rho}_i^a \hat{\pi}_i^a \hat{p}_{ij}^a$, and applied to vectors $\mathbf{\hat{h}}$ and $\mathbf{\hat{m}}$ results
in $[\mathbf{\mathbf{\hat{l}}}]_i  = \sum_{j \in \mathcal{T}} \sum_{a \in \mathcal{\hat{A}}_i} \hat{\rho}_i^a\hat{\pi}_i^a \hat{p}_{ij}^a\bar{r}_{ij}^a$ and $[\mathbf{\mathbf{\hat{o}}}]_i = \sum_{j \in \mathcal{N}} \sum_{a \in \mathcal{\hat{A}}_i} \hat{\rho}_i^a\hat{\pi}_i^a \hat{p}_{ij}^a\bar{r}_{ij}^a $ respectively. After simplification, we have $[\widehat{U}]_{ij} = \sum_{a \in \mathcal{\hat{A}}_i} \pi_i^a \hat{p}_{ij}^a$, $[\mathbf{\mathbf{\hat{l}}}]_{i} = \sum_{j \in \mathcal{T}} \sum_{a \in \mathcal{\hat{A}}_i} \pi_i^a \hat{p}_{ij}^a\bar{r}_{ij}^a$, and $[\mathbf{\mathbf{\hat{o}}}]_{i} = \sum_{j \in \mathcal{N}} \sum_{a \in \mathcal{\hat{A}}_i} \pi_i^a \hat{p}_{ij}^a\bar{r}_{ij}^a$. Using
these policy sampling error corrected quantities, we can derive the fixed-point for true policy, $\pi$, in a similar manner 
as earlier:

\begin{equation}
    v^{\pi}(i)  = \left[(I - \gamma \widehat{U})^{-1} (\mathbf{\mathbf{\hat{o}}} + \mathbf{\mathbf{\hat{l}}})\right]_i \label{eq:risceefixedpoint}
\end{equation}

In computing this new fixed-point, we have corrected for the policy sampling error, resulting
in a more accurate estimation of Equation (\ref{eq:truefixedpoint}) than Equation (\ref{eq:ceefixedpoint}). Now, the value function 
is computed for the true policy that we are interested in evaluating, $\pi$.

\section{Convergence of Batch Linear TD(0) to the MRP CE Fixed-Point}
\label{sec:mrp_td_convergence_proof}
\thmrptdconvergence*


\begin{proof}
Batch linear TD(0) makes an update
to the weight vector, $\bw_n$ (of dimension, length of the feature vector), after each presentation of the
batch:
\begin{align*}
    \bw_{n+1} &= \bw_n + \sum_{\tau \in \mathcal{D}} \sum_{t=1}^{L_\tau} \alpha \left[  (\bar{r}_t + \gamma \bw_n^T \bx_{t+1}) - \bw_n^T \bx_t\right] \bx_t
\end{align*}
where $\mathcal{D}$ is the batch of episodes, $L_\tau$ is the length of each
episode $\tau$,
and $\alpha$ is the learning rate.

We can re-write the whole presentation of the batch
of data in terms of the number of times there was a 
transition from state $i$ to state $j$ in the batch i.e.
$\hat{c}_{ij} = \hat{d}_i \hat{p}_{ij}$, where $\hat{d}_i$ is the
number of times state $i \in \widehat{\mathcal{N}}$ appears
in the batch.

\begin{flalign*}
    \bw_{n+1} &= \bw_n + \sum_{\tau \in \mathcal{D}} \sum_{t=1}^{L_\tau} \alpha \left[  (\bar{r}_t + \gamma \bw_n^T \bx_{t+1} - \bw_n^T \bx_t)\right] \bx_t\\
    &= \bw_n + \alpha\sum_{i \in \widehat{\mathcal{N}}} \sum_{j \in \widehat{\mathcal{N}} \cup \widehat{\mathcal{T}}} \hat{c}_{ij} \left[(\bar{r}_{ij} + \gamma \bw_n^T \bx_j - \bw_n^T \bx_i)\right] \bx_i \\
    &= \bw_n + \alpha\sum_{i \in \widehat{\mathcal{N}}} \sum_{j \in \widehat{\mathcal{N}} \cup \widehat{\mathcal{T}}} \hat{d}_i \hat{p}_{ij}  \left[(\bar{r}_{ij} + \gamma \bw_n^T \bx_j - \bw_n^T \bx_i)\right] \bx_i \\
    &= \bw_n + \alpha\sum_{i \in \widehat{\mathcal{N}}} \sum_{j \in \widehat{\mathcal{N}} \cup \widehat{\mathcal{T}}} \hat{d}_i \hat{p}_{ij} (\bar{r}_{ij} + \gamma \bw_n^T \bx_j) \bx_i - \alpha\sum_{i \in \widehat{\mathcal{N}}} \sum_{j \in \widehat{\mathcal{N}} \cup \widehat{\mathcal{T}}} \hat{d}_i \hat{p}_{ij} (\bw_n^T \bx_i) \bx_i\\
    &= \bw_n + \alpha\sum_{i \in \widehat{\mathcal{\widehat{\mathcal{N}}}}} \hat{d}_i \bx_i \sum_{j \in \widehat{\mathcal{N}} \cup \widehat{\mathcal{T}}} \hat{p}_{ij} (\bar{r}_{ij} + \gamma \bw_n^T \bx_j)  - \alpha\sum_{i \in \widehat{\mathcal{N}}} \hat{d}_i (\bw_n^T \bx_i)\bx_i  \sum_{j \in \widehat{\mathcal{N}} \cup \widehat{\mathcal{T}}} \hat{p}_{ij} \\
    &= \bw_n + \alpha\sum_{i \in \widehat{\mathcal{N}}} \hat{d}_i \bx_i\left[\sum_{j \in \widehat{\mathcal{N}} \cup \widehat{\mathcal{T}}} \hat{p}_{ij} (\bar{r}_{ij} + \gamma \bw_n^T \bx_j) -  \bw_n^T \bx_i \right] && \text{
    $\sum_{j \in \widehat{\mathcal{N}} \cup \widehat{\mathcal{T}}}  \hat{p}_{ij} = 1$}\\
    &= \bw_n + \alpha\sum_{i \in \widehat{\mathcal{N}}} \hat{d}_i \bx_i\left[\left(\sum_{j \in \widehat{\mathcal{N}}} \hat{p}_{ij} (\bar{r}_{ij} + \gamma \bw_n^T \bx_j)\right)
    + \left(\sum_{j \in \widehat{\mathcal{T}}} \hat{p}_{ij} \bar{r}_{ij}\right)-  \bw_n^T \bx_i \right] && \text{If $\bx_j \in \mathcal{\widehat{T}},\bw_n^T \bx_j = 0$} \\
    &= \bw_n + \alpha\sum_{i \in \widehat{\mathcal{N}}} \hat{d}_i \bx_i\left[\left(\sum_{j \in \widehat{\mathcal{N}}} \hat{p}_{ij} \bar{r}_{ij}\right)+ \left(\gamma\sum_{j \in \widehat{\mathcal{N}}} \hat{p}_{ij} \bw_n^T \bx_j\right)
    + \left(\sum_{j \in \widehat{\mathcal{T}}}  \hat{p}_{ij}\bar{r}_{ij}\right)  -  \bw_n^T \bx_i \right]
\end{flalign*}

\begin{align} \label{eqn:mrp_mat_form}
    \bw_{n+1} = \bw_n + \alpha \widehat{X}\widehat{D} \left[\mathbf{\hat{m}} + \gamma \widehat{Q}\widehat{X}^T\bw_n
    + \mathbf{\hat{h}} - \widehat{X}^T\bw_n \right]
\end{align}

where $\widehat{X}$ denotes the matrix (of dimensions, length of the feature vector by $|\widehat{\sset}|$) with columns, $\bx_i \in \widehat{\mathcal{S}}$ and $\widehat{D}$ is a diagonal matrix (of dimensions, $|\widehat{\sset}|$ by $|\widehat{\sset}|$) with $\widehat{D}_{ii} = \hat{d}_i$. Given the successive updates to the weight vector $\bw_n$, we now consider the actual values predicted as  the following by multiplying $\widehat{X}^T$ on both sides:

\begin{align*}
    \widehat{X}^T\bw_{n+1} &= \widehat{X}^T\bw_n + \alpha \widehat{X}^T \widehat{X}\widehat{D} \left( \mathbf{\hat{m}} + \mathbf{\hat{h}} + \gamma \widehat{Q}\widehat{X}^T\bw_n - \widehat{X}^T \bw_n\right) \\
    &= \widehat{X}^T\bw_n + \alpha \widehat{X}^T \widehat{X}\widehat{D}\left(\mathbf{\hat{m}} + \mathbf{\hat{h}}\right) + \alpha \widehat{X}^T \widehat{X}\widehat{D} \left(\gamma \widehat{Q}\widehat{X}^T\bw_n - \widehat{X}^T \bw_n\right)\\
    &= \alpha \widehat{X}^T \widehat{X}\widehat{D}\left(\mathbf{\hat{m}} + \mathbf{\hat{h}}\right) + \left(I - \alpha \widehat{X}^T \widehat{X}\widehat{D} \left(I - \gamma \widehat{Q}\right)\right)\widehat{X}^T\bw_n
\end{align*}
We then unroll the above equation by recursively applying $\widehat{X}^T\bw_{n}$ till $n = 0$.

\begin{align*}
    \widehat{X}^T\bw_{n+1} &= \alpha \widehat{X}^T \widehat{X}\widehat{D}\left(\mathbf{\hat{m}} + \mathbf{\hat{h}}\right) + \left(I - \alpha \widehat{X}^T \widehat{X}\widehat{D} \left(I - \gamma \widehat{Q}\right)\right)\alpha \widehat{X}^T \widehat{X}\widehat{D}\left(\mathbf{\hat{m}} + \mathbf{\hat{h}}\right) \\&+ \left(I - \alpha \widehat{X}^T \widehat{X}\widehat{D} \left(I - \gamma \widehat{Q}\right)\right)^2\widehat{X}^T\bw_{n - 1} \\
    & \vdots\\
    &= \sum_{k = 0}^{n - 1} \left(I - \alpha \widehat{X}^T \widehat{X}\widehat{D} \left(I - \gamma \widehat{Q}\right)\right)^k\alpha \widehat{X}^T \widehat{X}\widehat{D}\left(\mathbf{\hat{m}} + \mathbf{\hat{h}}\right) \\&+\left(I - \alpha \widehat{X}^T \widehat{X}\widehat{D} \left(I - \gamma \widehat{Q}\right)\right)^n\widehat{X}^T\bw_0 \numberthis \label{eq:conv_val}
\end{align*}

Assuming that as $n \rightarrow \infty$, $ (I - \alpha \widehat{X}^T \widehat{X} \widehat{D} (I - \gamma \widehat{Q}))^n \rightarrow 0$, we can drop the second term and the sequence $\{\widehat{X}^T\bw_n\}$ converges to:
\begin{align*}
    \lim_{n\to \infty}\widehat{X}^T \bw_n &= \left(I - (I - \alpha \widehat{X}^T \widehat{X} \widehat{D} (I - \gamma \widehat{Q}))\right)^{-1} (\alpha \widehat{X}^T \widehat{X}\widehat{D} (\mathbf{\hat{m}} + \mathbf{\hat{h}}))\\
    &= (I - \gamma \widehat{Q})^{-1} \widehat{D}^{-1} (\widehat{X}^T\widehat{X})^{-1} \alpha^{-1} \alpha \widehat{X}^T \widehat{X} \widehat{D} (\mathbf{\hat{m}} + \mathbf{\hat{h}}) \\
    &= (I - \gamma \widehat{Q})^{-1}(\mathbf{\hat{m}} + \mathbf{\hat{h}})\\
\lim_{n\to \infty}\mathbb{E}\left[\bx_i^T\bw_n\right] &= \left[(I - \gamma \widehat{Q})^{-1}(\mathbf{\hat{m}} + \mathbf{\hat{h}})\right]_i, \forall i\in \mathcal{\widehat{N}}
\end{align*}

What is left to show now is $n \rightarrow \infty$, $ (I - \alpha \widehat{X}^T \widehat{X} \widehat{D} (I - \gamma \widehat{Q}))^n \rightarrow 0$.
Following \citet{sutton1988learning}, we first show that $\widehat{D} (I - \gamma \widehat{Q})$ is positive definite, and then that $\widehat{X}^T \widehat{X} \widehat{D} (I - \gamma \widehat{Q})$ has a full set of eigenvalues all of whose real parts are positive.
This enables us to show that $\alpha$ can be chosen so that eigenvalues of $ (I - \alpha \widehat{X}^T \widehat{X} \widehat{D} (I - \gamma \widehat{Q}))$ are less than $1$ in modulus, which assures us that its powers converge to $0$.





To show that  $\widehat{D} (I - \gamma \widehat{Q})$ is positive definite, we refer to the 
Gershgorin Circle theorem \citep{izvestija/gerschgorin31}, which states that if a matrix, $A$, is real, symmetric,
and strictly diagonally dominant with positive diagonal entries, then $A$ is positive definite. However, we cannot
apply this theorem as is to $\widehat{D} (I - \gamma \widehat{Q})$ since the matrix is not necessarily symmetric. To use
the theorem, we first apply another theorem (Theorem A.3 from \citet{sutton1988learning}) that states: a square matrix $A$ is positive definite if and only if $A + A^T$
is positive definite. So it suffices to show that $\widehat{D} (I - \gamma \widehat{Q}) + (\widehat{D} (I - \gamma \widehat{Q}))^T$
is positive definite.

Consider the matrix $S = \widehat{D} (I - \gamma \widehat{Q}) + (\widehat{D} (I - \gamma \widehat{Q}))^T$. We know that
$S$ is real and symmetric. It remains to show that the diagonal entries are positive and that $S$ is strictly diagonally
dominant. First, we look at the diagonal entries,  $S_{ii} = 2[\widehat{D}(I - \gamma \widehat{Q})]_{ii} = 
2\hat{d}_i(1 - \gamma\hat{p}_{ii}) > 0, \forall i\in \widehat{\mathcal{N}}$, which are positive. Second, we have the non-diagonal
entries for $i \not= j$ as $S_{ij} = [\widehat{D}(I - \gamma \widehat{Q})]_{ij} + [\widehat{D}(I - \gamma \widehat{Q})]_{ji}
= -\gamma\hat{d}_i\hat{p}_{ij} - \gamma\hat{d}_j\hat{p}_{ji} \leq 0$, which are nonpositive. We want to show that $|S_{ii}| \geq\sum_{j \neq i}|S_{ij}|$, with strict inequality holding for at least one $i$;  we know that the diagonal elements $S_{ii} > 0$ and non-diagonal elements  $S_{ij} \leq 0$, $i\not=j$. Hence,
to show that $S$ is strictly diagonally dominant, it is enough to show that $S_{ii} > -\sum_{j \neq i}S_{ij}$, which means
we can simply show that the sum of each entire row is greater than $0$, i.e. $\sum_j S_{ij} > 0$.

Before we show that $\sum_j S_{ij} > 0$, we note that $\hat{d}^T = \hat{\mu}^T (I - \widehat{Q})^{-1}$ where
$\hat{\mu}_i$ is the empirical state distribution of state $i$. Given the definitions
of $\hat{d}$, $\hat{\mu}$, and $\widehat{Q}$, this fact follows from \citet{kemeny1960finite} and is used by
 \citet{sutton1988learning}. Using this fact, we show that $\sum_j S_{ij} \geq 0$:

\begin{align*}
    \sum_{j} S_{ij} &= \sum_{j} \left([\widehat{D}(I - \gamma \widehat{Q})]_{ij} + [\widehat{D}(I - \gamma \widehat{Q})]^T_{ij}\right) \\
    &= \hat{d}_i\sum_{j}([I - \gamma \widehat{Q}]_{ij} + \sum_j \hat{d}_j[I - \gamma \widehat{Q}]^T_{ij}) \\
    &= \hat{d}_i\sum_{j}(1 - \gamma \hat{p}_{ij} ) + \left[\hat{d}^T (I - \widehat{Q})\right]_i \\
    &= \hat{d}_i\sum_{j}(1 - \gamma p_{ij} ) + \left[\hat{\mu}^T (I - \widehat{Q})^{-1}(I - \widehat{Q})\right]_i && \text{$\hat{d}^T = \hat{\mu}^T (I - \widehat{Q})^{-1}$}\\
    &= \hat{d}_i(1 - \gamma \sum_j \hat{p}_{ij}) + \hat{\mu}_i \\
    &\geq 0,
\end{align*}
where the final inequality is strict since $\hat{\mu}$  is positive for at least one element. Given the above, we have
shown that $S$ is real, symmetric, and strictly diagonally dominant; hence, $S$ is positive definite according to the
Gershgorin Circle theorem \citep{izvestija/gerschgorin31}. Since, $S
= \widehat{D} (I - \gamma \widehat{Q}) + (\widehat{D} (I - \gamma \widehat{Q}))^T$ is
positive definite, we have $\widehat{D} (I - \gamma \widehat{Q})$ to be positive definite.


Now we need to show that $\widehat{X}^T\widehat{X}\widehat{D} (I - \gamma \widehat{Q})$ has a full set of eigenvalues, all of whose real parts are positive. We know that $\widehat{X}^T\widehat{X}\widehat{D} (I - \gamma \widehat{Q})$ has a full
set of eigenvalues for the same reason shown by \citet{sutton1988learning}, i.e.  $\widehat{X}^T\widehat{X}\widehat{D} (I - \gamma \widehat{Q})$ is a product of three non-singular matrices, which means $\widehat{X}^T\widehat{X}\widehat{D} (I - \gamma \widehat{Q})$ is nonsingular as well; hence, no eigenvalues are $0$ i.e. its set of eigenvalues is full.

Consider $\lambda$ and $y$ to be an eigevalue and eigenvector pair of $\widehat{X}^T\widehat{X}\widehat{D} (I - \gamma \widehat{Q})$. First lets consider that $y$ may be a complex number and is of the form $y = a + bi$, and let $z = (\widehat{X}^T\widehat{X})^{-1}y, y\neq 0$. Second, we consider $\widehat{D} (I - \gamma \widehat{Q})$  from earlier i.e. where
$*$ is the conjugate-transpose:

\begin{align*}
    y^*\widehat{D} (I - \gamma \widehat{Q})y & = z^* \widehat{X}^T \widehat{X} \widehat{D} (I - \gamma \widehat{Q})y && \text{substituting $y^*$}\\
   &= z^*\lambda y && \text{$\widehat{X}^T\widehat{X}\widehat{D} (I - \gamma \widehat{Q})y = \lambda y$}\\
   &= \lambda z^*\widehat{X}^T\widehat{X}z && \text{substituting $y$}\\
   &= \lambda (\widehat{X}z)^*\widehat{X}z\\
   (a^T - b^Ti )(\widehat{D} (I - \gamma \widehat{Q}))(a^T + b^Ti )&= \lambda (\widehat{X}z)^*\widehat{X}z && \text{substituting
   $y^*$ and $y$}
\end{align*}
From the above equality, we know that the real parts (Re) of the LHS and RHS are equal as well i.e.
\begin{align*}
    \text{Re}\left(y^*\widehat{D} (I - \gamma \widehat{Q})y\right) &= \text{Re}\left(\lambda (\widehat{X}z)^*\widehat{X}z\right) \\   
    a^T\widehat{D} (I - \gamma \widehat{Q})a + b^T\widehat{D} (I - \gamma \widehat{Q})b &= (\widehat{X}z)^*\widehat{X}z \text{Re}\left(\lambda \right)
\end{align*}

LHS must be strictly positive since we already proved that $\widehat{D} (I - \gamma \widehat{Q})$ is positive definite and
by definition,  RHS, $(\widehat{X}z)^*\widehat{X}z$, is strictly positive as well. Thus, the Re($\lambda$) must be positive. 
Finally, using this result we want to show that the eigenvalues of 
$ (I - \alpha \widehat{X}^T \widehat{X} \widehat{D} (I - \gamma \widehat{Q}))$
are of modulus less than $1$ for a suitable $\alpha$.

First, we can see that $y$ is also an eigenvector of $ (I - \alpha \widehat{X}^T \widehat{X} \widehat{D} (I - \gamma \widehat{Q}))$, since $ (I - \alpha \widehat{X}^T \widehat{X} \widehat{D} (I - \gamma \widehat{Q}))y = y - \alpha\lambda y
= (1 - \lambda\alpha)y$, where $\lambda' = (1 - \alpha\lambda)$ is an eigenvalue of $ (I - \alpha \widehat{X}^T \widehat{X} \widehat{D} (I - \gamma \widehat{Q}))$. Second, we want to find suitable $\alpha$ such that the modulus of
 $\lambda'$ is less than $1$. We have the modulus of $\lambda'$:
 
 \begin{align*}
    \| \lambda' \| &= \|1 - \alpha \lambda \| \\
    &= \sqrt{(1 - \alpha a)^2 + (-\alpha b^2)}  && \text{substituting $\lambda  = a + bi$ of general complex form}\\
    &= \sqrt{1 - 2 \alpha a + \alpha^2 a^2 + \alpha^2 b^2} \\
    &= \sqrt{1 - 2 \alpha a + \alpha^2 (a^2 + b^2)} \\
    &< \sqrt{1 - 2 \alpha a + \alpha \frac{2a}{(a^2 + b^2)} (a^2 + b^2)} && \text{using $\alpha = \frac{2a}{(a^2 + b^2)}$}\\
    &= \sqrt{1 - 2 \alpha a + 2 \alpha a} = 1
\end{align*}

From above, we can see that if $\alpha$ is chosen such that $0 < \alpha < \frac{2a}{a^2 + b^2}$, then $\lambda'$ will have modulus less than $1$. Then using the theorem that states: if a matrix $A$ has $n$ independent eigenvectors with eigenvalues $\lambda_i$,
then $A^k \to 0$ as $k \to \infty$ if and only if all $\|\lambda_i \|< 1$, which implies that $\lim_{n \to \infty} \left(I - \alpha \widehat{X}\widehat{D}(I - \widehat{Q})\widehat{X}^T\right)^n = 0$, taking the trailing element in Equation (\ref{eq:conv_val}) to $0$ for a suitable $\alpha$. We thus prove convergence to the fixed point in  Equation (\ref{eq:mrp-dp}) if a batch linear TD(0) update is used with an appropriate step size $\alpha$.
\end{proof}

\section{Convergence of Batch Linear TD(0) to the MDP CE Fixed-Point}
\label{sec:mpd_td_convergence_proof}
\thmdptdconvergence*

\begin{proof}
Batch linear TD(0) makes an update
to weight vector, $\bw_n$ (of dimension, length of the feature vector), after each presentation of the
batch:
\begin{align*}
    \bw_{n+1} &= \bw_n + \sum_{\tau \in \mathcal{D}} \sum_{t=1}^{L_\tau} \alpha \left[  (\bar{r}_t + \gamma \bw_n^T \bx_{t+1}) - \bw_n^T \bx_t\right] \bx_t
\end{align*}
where $\mathcal{D}$ is the batch of episodes, $L_\tau$ is the length of each
episode $\tau$,
and $\alpha$ is the learning rate.

We can re-write the whole presentation of the batch
of data in terms of the number of times there was a 
transition from state $i$ to state $j$ when taking
action $a$ in the batch i.e. $\hat{c}_{ij}^a = \hat{d}_i \hat{\pi}_i^a \hat{p}_{ij}^a$, where $\hat{d}_i$ is the
number of times state $i \in \widehat{\mathcal{N}}$ appears
in the batch.

\begin{flalign*}
    \bw_{n+1} &= \bw_n + \sum_{\tau \in \mathcal{D}} \sum_{t=1}^{L_\tau} \alpha \left[  (\bar{r}_t + \gamma \bw_n^T \bx_{t+1} - \bw_n^T \bx_t)\right] \bx_t &\\
    &= \bw_n + \alpha\sum_{i \in \widehat{\mathcal{N}}} \sum_{j \in \widehat{\mathcal{N}} \cup \widehat{\mathcal{T}}} \sum_{a\in \mathcal{\hat{A}}_i} \hat{c}_{ij}^a \left[(\bar{r}_{ij}^a + \gamma \bw_n^T \bx_j - \bw_n^T \bx_i)\right] \bx_i &\\
    &= \bw_n + \alpha\sum_{i \in \widehat{\mathcal{N}}} \sum_{j \in \widehat{\mathcal{N}} \cup \widehat{\mathcal{T}}} \sum_{a \in \mathcal{\widehat{A}}_i} \hat{d}_i \hat{\pi}_i^a \hat{p}_{ij}^a \left[(\bar{r}_{ij}^a + \gamma \bw_n^T \bx_j - \bw_n^T \bx_i)\right] \bx_i\\
    &= \bw_n + \alpha\sum_{i \in \widehat{\mathcal{N}}} \sum_{j \in \widehat{\mathcal{N}} \cup \widehat{\mathcal{T}}} \sum_{a \in \mathcal{\hat{A}}_i} \hat{d}_i \hat{p}_{ij}^a \hat{\pi}_i^a (\bar{r}_{ij}^a + \gamma \bw_n^T \bx_j) \bx_i - \alpha\sum_{i \in \widehat{\mathcal{N}}} \sum_{j \in \widehat{\mathcal{N}} \cup \widehat{\mathcal{T}}} \sum_{a \in \mathcal{\hat{A}}_i} \hat{d}_i \hat{\pi}_i^a \hat{p}_{ij}^a (\bw_n^T \bx_i) \bx_i\\
    &= \bw_n + \alpha\sum_{i \in \widehat{\mathcal{\widehat{\mathcal{N}}}}} \hat{d}_i \bx_i \sum_{j \in \widehat{\mathcal{N}} \cup \widehat{\mathcal{T}}} \sum_{a\in \mathcal{\hat{A}}_i}  \hat{p}_{ij}^a \hat{\pi}_i^a(\bar{r}_{ij}^a + \gamma \bw_n^T \bx_j)  - \alpha\sum_{i \in \widehat{\mathcal{N}}} \hat{d}_i (\bw_n^T \bx_i)\bx_i  \sum_{j \in \widehat{\mathcal{N}} \cup \widehat{\mathcal{T}}} \sum_{a\in \mathcal{\hat{A}}_i}  \hat{\pi}_i^a \hat{p}_{ij}^a
    \end{flalign*}
    \begin{align*}
    &= \bw_n + \alpha\sum_{i \in \widehat{\mathcal{N}}} \hat{d}_i \bx_i\left[\sum_{j \in \widehat{\mathcal{N}} \cup \widehat{\mathcal{T}}} \sum_{a \in \mathcal{\hat{A}}_i}  \hat{p}_{ij}^a \hat{\pi}_i^a(\bar{r}_{ij}^a + \gamma \bw_n^T \bx_j) -  \bw_n^T \bx_i \right] && \text{
    $\sum_{j \in \widehat{\mathcal{N}} \cup \widehat{\mathcal{T}}} \sum_{a\in \mathcal{\hat{A}}_i} \hat{\pi}_i^a \hat{p}_{ij}^a = 1$}\\
    &= \bw_n + \alpha\sum_{i \in \widehat{\mathcal{N}}} \hat{d}_i \bx_i\left[\left(\sum_{j \in \widehat{\mathcal{N}}} \sum_{a\in \mathcal{\hat{A}}_i}  \hat{p}_{ij}^a \hat{\pi}_i^a(\bar{r}_{ij}^a + \gamma \bw_n^T \bx_j)\right)
    + \left(\sum_{j \in \widehat{\mathcal{T}}} \sum_{a \in \mathcal{\hat{A}}_i}  \hat{p}_{ij}^a \hat{\pi}_i^a\bar{r}_{ij}^a\right)  -  \bw_n^T \bx_i \right]
    && \text{If $\bx_j \in \mathcal{\widehat{T}},\bw_n^T \bx_j = 0$}
    \end{align*}
    \begin{align*}
    &= \bw_n + \alpha\sum_{i \in \widehat{\mathcal{N}}} \hat{d}_i \bx_i\left[\left(\sum_{j \in \widehat{\mathcal{N}}} \sum_{a\in \mathcal{\hat{A}}_i}  \hat{p}_{ij}^a \hat{\pi}_i^a\bar{r}_{ij}^a\right) + \left(\gamma\sum_{j \in \widehat{\mathcal{N}}} \sum_{a\in \mathcal{\hat{A}}_i }  \hat{p}_{ij}^a \hat{\pi}_i^a \bw_n^T \bx_j\right)
    + \left(\sum_{j \in \widehat{\mathcal{T}}} \sum_{a\in \mathcal{\hat{A}}_i}  \hat{p}_{ij}^a \hat{\pi}_i^a \bar{r}_{ij}^a\right)  -  \bw_n^T \bx_i \right]
\end{align*}
\begin{align} \label{eqn:mdp_mat_form}
    \bw_{n+1} = \bw_n + \alpha \widehat{X}\widehat{D} \left[\mathbf{\hat{m}} + \gamma \widehat{Q}\widehat{X}^T\bw_n
    + \mathbf{\hat{h}} - \widehat{X}^T\bw_n \right]
\end{align}

where $\widehat{X}$ denotes the matrix (of dimensions, length of the feature vector by $|\widehat{\sset}|$) with columns, $\bx_i \in \widehat{\mathcal{S}}$  and $\widehat{D}$ is a diagonal matrix (of dimensions, $|\widehat{\sset}|$ by $|\widehat{\sset}|$) with $\widehat{D}_{ii} = \hat{d}_i$.

Notice that Equation (\ref{eqn:mdp_mat_form}) is the same as Equation (\ref{eqn:mrp_mat_form}) since the considered
MRP and MDP settings are equivalent. Due to this
similarity, we omit the proof from here below as it
is identical to the Theorem \ref{th:mrp_td_convergence} proof.

\end{proof}

\section{Convergence of Batch Linear PSEC-TD(0) to the PSEC-MDP-CE Fixed-Point}
\label{sec:ris_td_convergence_proof}
\subsection{PSEC Correction Applied to the New Estimate}
\label{sec:psec_supp_new_est}

We now show that batch linear PSEC-TD(0) converges to the policy corrected MDP-CE established in Equation (\ref{eq:risceefixedpoint}), which is equivalent to Equation (\ref{eq:ris-mdp-dp})

\thristdconvergence*
\begin{proof}
The proof for PSEC-TD(0) follows in large part the structure of the proof for TD(0).
Below we highlight the salient points in the proof.

Batch linear PSEC-TD(0) makes an update
to the weight vector, $\bw_n$ (of dimension, length of the feature vector), after each presentation of the
batch:
\begin{align*}
    \bw_{n+1} &= \bw_n + \sum_{\tau \in \mathcal{D}} \sum_{t=1}^{L_\tau} \alpha  \left[ \hat{\rho}_t (\bar{r}_t + \gamma \bw_n^T \bx_{t+1}) - \bw_n^T \bx_t\right] \bx_t
\end{align*}
where $\mathcal{D}$ is the batch of episodes, $L_\tau$ is the length of each
episode $\tau$, $\hat{\rho}_t$
is the PSEC correction weight at time $t$ for a given episode
$\tau$, and $\alpha$ is the learning rate. 





We can re-write the whole presentation of the batch
of data in terms of the number of times there was a 
transition from state $i$ to state $j$ when taking
action $a$ in the batch i.e. $\hat{c}_{ij}^a = \hat{d}_i \hat{\pi}_i^a \hat{p}_{ij}^a$, where $\hat{d}_i$ is the
number of times state $i \in \widehat{\mathcal{N}}$ appears
in the batch.

\begin{align*}
    \bw_{n+1} &= \bw_n + \sum_{\tau \in \mathcal{D}} \sum_{t=1}^{L_\tau} \alpha  \left[ \hat{\rho}_t (\bar{r_t} + \gamma \bw_n^T \bx_{t+1}) - \bw_n^T \bx_t\right] \bx_t\\
    &= \bw_n + \alpha\sum_{i \in \widehat{\mathcal{N}}} \sum_{j \in \widehat{\mathcal{N}} \cup \widehat{\mathcal{T}}} \sum_{a\in \mathcal{\hat{A}}_i} \hat{c}_{ij}^a  \left[ \hat{\rho}_i^a (\bar{r}_{ij}^a + \gamma \bw_n^T \bx_j) - \bw_n^T \bx_i\right] \bx_i \\
    &= \bw_n + \alpha\sum_{i \in \widehat{\mathcal{N}}} \sum_{j \in \widehat{\mathcal{N}} \cup \widehat{\mathcal{T}}} \sum_{a \in \mathcal{\hat{A}}_i} \hat{d}_i \hat{\pi}_i^a \hat{p}_{ij}^a  \left[ \hat{\rho}_i^a (\bar{r}_{ij}^a + \gamma \bw_n^T \bx_j) - \bw_n^T \bx_i\right] \bx_i\\
    &= \bw_n + \alpha\sum_{i \in \widehat{\mathcal{N}}} \sum_{j \in \widehat{\mathcal{N}} \cup \widehat{\mathcal{T}}} \sum_{a \in \mathcal{\widehat{A}}_i} \hat{d}_i \hat{\pi}_i^a \hat{p}_{ij}^a  \left[ \frac{\pi_i^a}{\hat{\pi}_i^a} (\bar{r}_{ij}^a + \gamma \bw_n^T \bx_j) - \bw_n^T \bx_i\right] \bx_i\\
    &= \bw_n + \alpha\sum_{i \in \widehat{\mathcal{N}}} \sum_{j \in \widehat{\mathcal{N}} \cup \widehat{\mathcal{T}}} \sum_{a \in \mathcal{\hat{A}}_i} \hat{d}_i \hat{p}_{ij}^a \pi_i^a(\bar{r}_{ij}^a + \gamma \bw_n^T \bx_j) \bx_i - \alpha\sum_{i \in \widehat{\mathcal{N}}} \sum_{j \in \widehat{\mathcal{N}} \cup \widehat{\mathcal{T}}} \sum_{a \in \mathcal{\hat{A}}_i} \hat{d}_i \hat{\pi}_i^a \hat{p}_{ij}^a (\bw_n^T \bx_i) \bx_i\\
    &= \bw_n + \alpha\sum_{i \in \widehat{\mathcal{\widehat{\mathcal{N}}}}} \hat{d}_i \bx_i \sum_{j \in \widehat{\mathcal{N}} \cup \widehat{\mathcal{T}}} \sum_{a\in \mathcal{\hat{A}}_i}  \hat{p}_{ij}^a \pi_i^a(\bar{r}_{ij}^a + \gamma \bw_n^T \bx_j)  - \alpha\sum_{i \in \widehat{\mathcal{N}}} \hat{d}_i (\bw_n^T \bx_i)\bx_i  \sum_{j \in \widehat{\mathcal{N}} \cup \widehat{\mathcal{T}}} \sum_{a\in \mathcal{\hat{A}}_i}  \hat{\pi}_i^a \hat{p}_{ij}^a
    \end{align*}
    \begin{align*}
    &= \bw_n + \alpha\sum_{i \in \widehat{\mathcal{N}}} \hat{d}_i \bx_i\left[\sum_{j \in \widehat{\mathcal{N}} \cup \widehat{\mathcal{T}}} \sum_{a \in \mathcal{\hat{A}}_i}  \hat{p}_{ij}^a \pi_i^a(\bar{r}_{ij}^a + \gamma \bw_n^T \bx_j) -  \bw_n^T \bx_i \right] && \text{
    $\sum_{j \in \widehat{\mathcal{N}} \cup \widehat{\mathcal{T}}} \sum_{a\in \mathcal{\hat{A}}_i}  \hat{\pi}_i^a \hat{p}_{ij}^a = 1$}\\
    &= \bw_n + \alpha\sum_{i \in \widehat{\mathcal{N}}} \hat{d}_i \bx_i\left[\left(\sum_{j \in \widehat{\mathcal{N}}} \sum_{a\in \mathcal{\hat{A}}_i}  \hat{p}_{ij}^a \pi_i^a(\bar{r}_{ij}^a + \gamma \bw_n^T \bx_j)\right)
    + \left(\sum_{j \in \widehat{\mathcal{T}}} \sum_{a \in \mathcal{\hat{A}}_i}  \hat{p}_{ij}^a \pi_i^a\bar{r}_{ij}^a\right)  -  \bw_n^T \bx_i \right]
    && \text{If $\bx_j \in \mathcal{\widehat{T}},\bw_n^T \bx_j = 0$}
    \end{align*}
    \begin{align*}
    &= \bw_n + \alpha\sum_{i \in \widehat{\mathcal{N}}} \hat{d}_i \bx_i\left[\left(\sum_{j \in \widehat{\mathcal{N}}} \sum_{a\in \mathcal{\hat{A}}_i}  \hat{p}_{ij}^a \pi_i^a\bar{r}_{ij}^a\right) + \left(\gamma\sum_{j \in \widehat{\mathcal{N}}} \sum_{a\in \mathcal{\hat{A}}_i }  \hat{p}_{ij}^a \pi_i^a \bw_n^T \bx_j\right)
    + \left(\sum_{j \in \widehat{\mathcal{T}}} \sum_{a\in \mathcal{\hat{A}}_i}  \hat{p}_{ij}^a \pi_i^a\bar{r}_{ij}^a\right)  -  \bw_n^T \bx_i \right]\\
    &= \bw_n + \alpha \widehat{X}\widehat{D} \left[\mathbf{\hat{o}} + \gamma \widehat{U}\widehat{X}^T\bw_n
    + \mathbf{\hat{l}} - \widehat{X}^T\bw_n \right]
\end{align*}

where $\widehat{X}$ denotes the matrix (of dimensions, length of the feature vector by $|\widehat{\sset}|$) with columns, $\bx_i \in \widehat{\mathcal{S}}$  and $\widehat{D}$ is a diagonal matrix (of dimensions, $|\widehat{\sset}|$ by $|\widehat{\sset}|$) with $\widehat{D}_{ii} = \hat{d}_i$.

Assuming that as $n \rightarrow \infty$, $ (I - \alpha \widehat{X}^T \widehat{X} \widehat{D} (I - \gamma \widehat{U}))^n \rightarrow 0$, we can drop the second term and the sequence $\{\widehat{X}^T\bw_n\}$ converges to:
\begin{align*}
    \lim_{n\to \infty}\widehat{X}^T \bw_n &= \left(I - (I - \alpha \widehat{X}^T \widehat{X} \widehat{D} (I - \gamma \widehat{U}))\right)^{-1} (\alpha \widehat{X}^T \widehat{X}\widehat{D} (\mathbf{\hat{o}} + \mathbf{\hat{l}}))\\
    &= (I - \gamma \widehat{U})^{-1} \widehat{D}^{-1} (\widehat{X}^T\widehat{X})^{-1} \alpha^{-1} \alpha \widehat{X}^T \widehat{X} \widehat{D} (\mathbf{\hat{o}} + \mathbf{\hat{l}}) \\
    &= (I - \gamma \widehat{U})^{-1}(\mathbf{\hat{o}} + \mathbf{\hat{l}})\\
\lim_{n\to \infty}\mathbb{E}\left[\bx_i^T\bw_n\right] &= \left[(I - \gamma \widehat{U})^{-1}(\mathbf{\hat{o}} + \mathbf{\hat{l}})\right]_i, \forall i\in \mathcal{\hat{N}}
\end{align*}

What is left to show now is that as $n \rightarrow \infty$, $ (I - \alpha \widehat{X}^T \widehat{X} \widehat{D} (I - \gamma \widehat{U}))^n \rightarrow 0$, which we can show by following the steps shown for Equation (\ref{eq:conv_val}). Thus we prove convergence to the fixed-point (\ref{eq:ris-mdp-dp}).

\end{proof}
\subsection{Convergence to the MDP True Fixed-Point with Infinite Data}

With batch linear PSEC-TD(0) we have corrected for the
policy sampling error in batch linear TD(0). The remaining
inaccuracy of the policy sampling corrected 
certainty-equivalence fixed-point is due to the 
transition dynamics sampling error. In a model-free
setting, however, we cannot correct for this error in
the same way we corrected the policy sampling error.

We argue that as the batch size approaches infinite, the
maximum-likelihood estimate of the transition dynamics
will approach the true transition dynamics i.e. $\hat{p} \to p$. It then follows that
in expectation, the true value function will be reached. Thus,
the batch linear PSEC-TD(0) with an infinite batch size
will correctly converge to the true value function
fixed-point given by Equation (\ref{eq:truefixedpoint}).

\section{Additional Empirical Results}

In this section, we include additional results that were omitted in the main text due to space constraints.

\subsection{Tabular Setting: Discrete States and Actions}
\label{app:gridworld_add_results}
\subsubsection{Off-policy Results}
For off-policy TD(0), we always use the variant that applies the importance weight to the TD-error.

\begin{figure}[h]
    \centering
        \subfigure[Off-Policy TD]{\label{fig:off_grid}\includegraphics[scale=0.175]{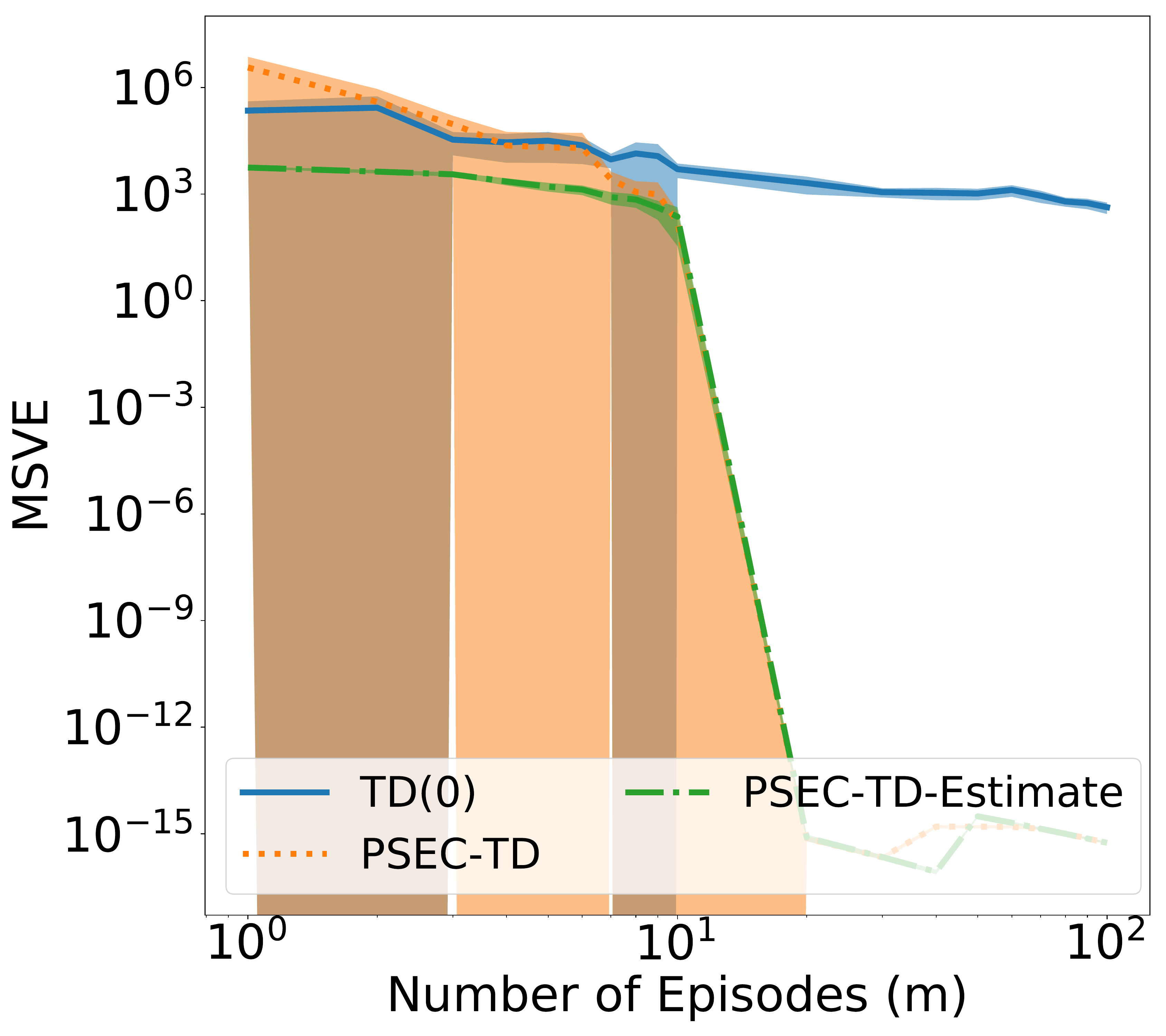}}
        \subfigure[Off-Policy LSTD]{\label{fig:off_lstd}\includegraphics[scale=0.175]{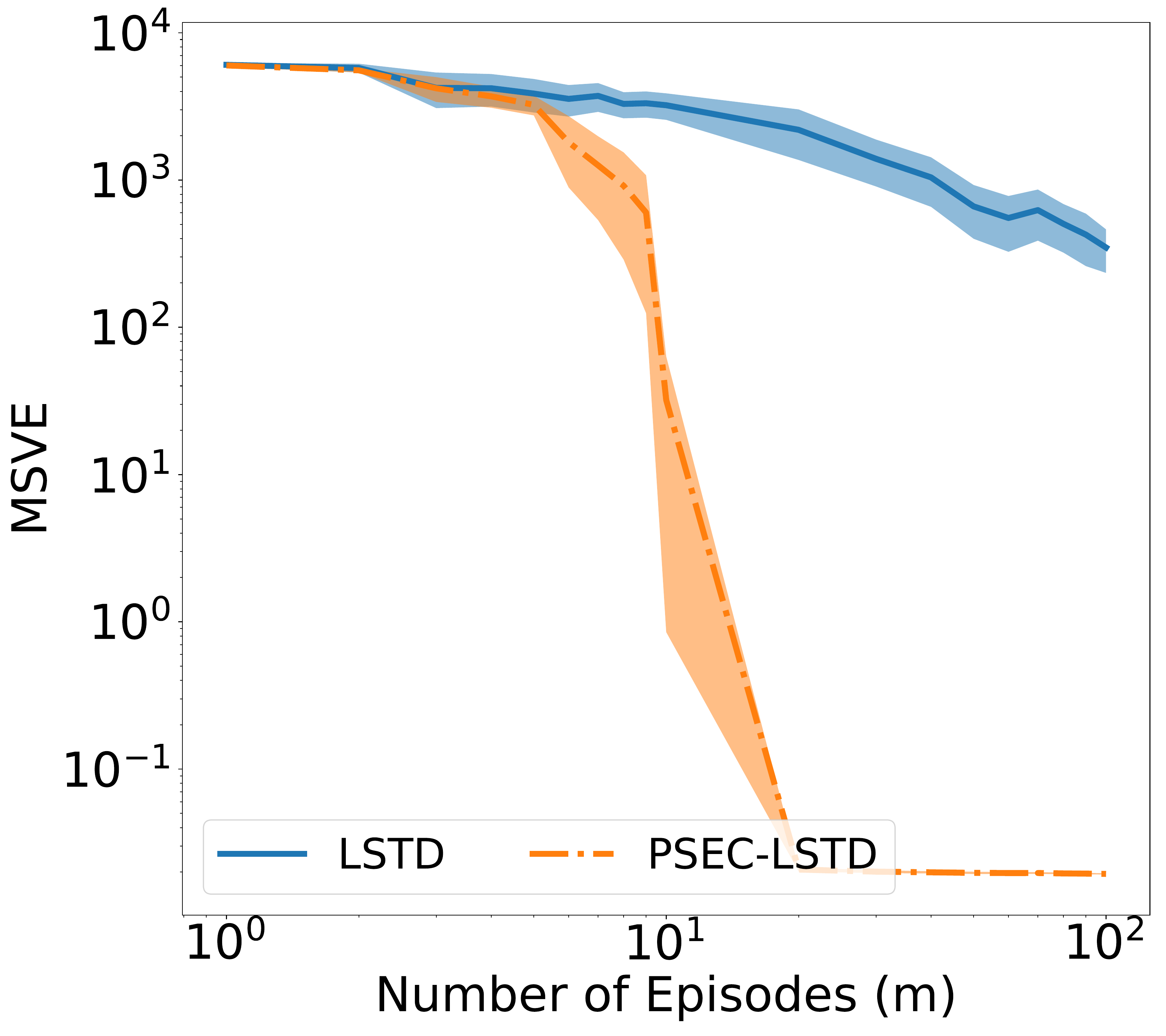}}
    \caption{\footnotesize Deterministic Gridworld  experiments.
    Both axes are log-scaled, leading to the asymmetric confidence intervals. Errors are computed over $50$ trials with $95\%$ confidence intervals.
   Figure \ref{fig:off_grid} and Figure \ref{fig:off_lstd}
    compare the final errors achieved by variants of PSEC-TD(0) and TD(0), and PSEC-LSTD(0) and LSTD(0) 
    respectively for varying batch size in the  off-policy case.}
\end{figure}

For off-policy TD(0) we only consider the PSEC-TD variant
as we found multiplying the new estimate by the weight was divergent. All methods show higher variance for the off-policy setting, however, PSEC-TD(0) variants still provide more accurate value function estimates.

\subsubsection{Effect of Environment Stochasticity}
\begin{figure}[H]
    \centering
        \includegraphics[scale=0.175]{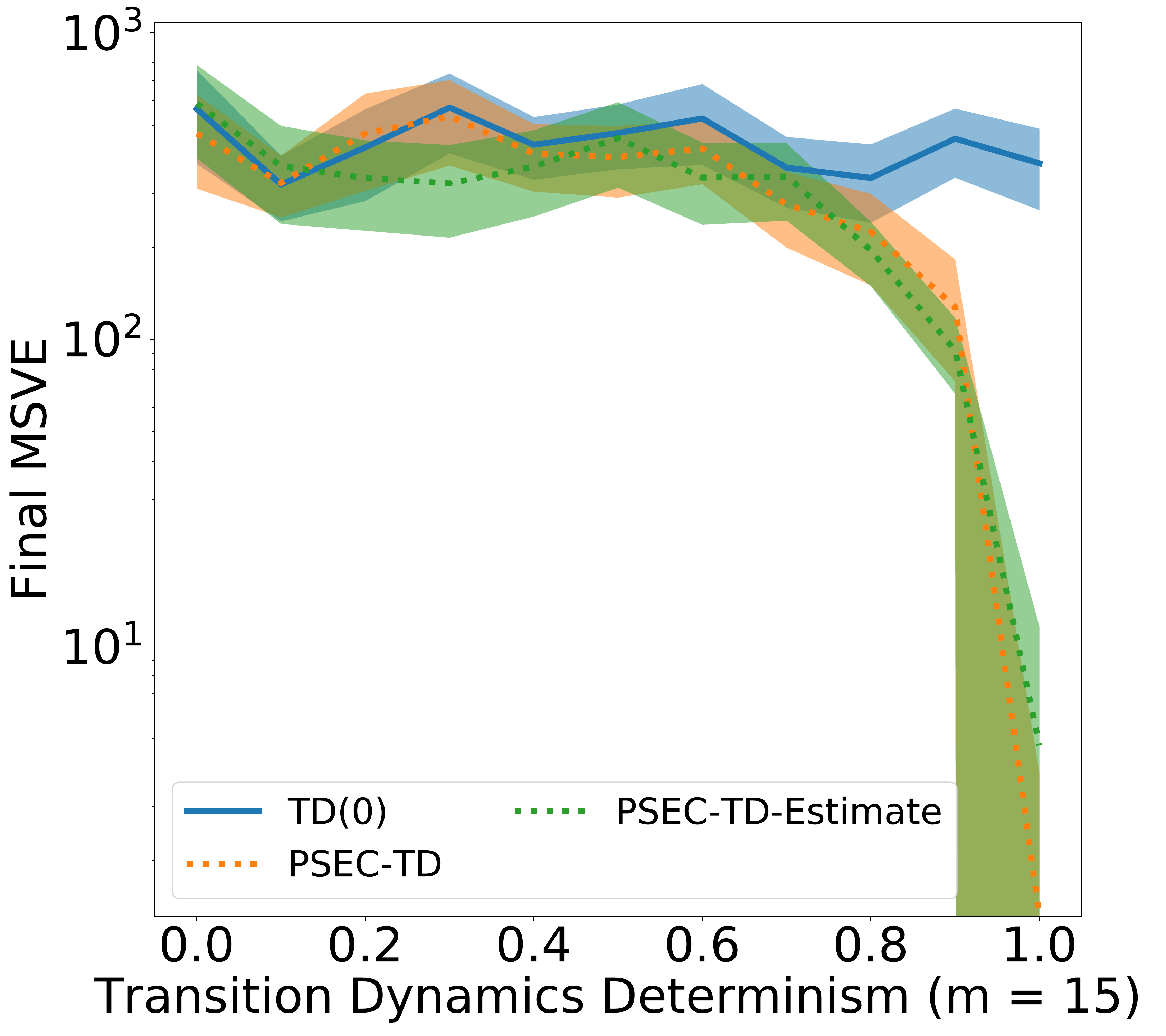}
    \caption{\footnotesize Additional Gridworld experiments. Errors  are computed over $50$ trials with $95\%$ confidence intervals. Figure \ref{fig:gridworld_stoch} is
    a $y$-axis log scaled graph that shows the final error (averaged over $100$ trials) achieved by the two variants of PSEC-TD(0) and TD(0)
    for a given batch size ($15$ episodes) with varying levels of determinism
    of the transition dynamics.}
    \label{fig:gridworld_stoch}
\end{figure}

According
to Theorem \ref{th:mdp_td_convergence}, TD suffers from policy \textit{and}
transition dynamics sampling error. We study this
observation through Figure \ref{fig:gridworld_stoch}, which
illustrates how the performance of PSEC
changes with different levels of
transition dynamics determinism for a fixed batch
size. In
Gridworld, the
determinism is varied according to a parameter, $p$,
where the environment becomes purely deterministic
or stochastic as $p \to 1$ or $p \to 0$ respectively.  From 
Theorem \ref{th:ris_td_convergence} we expect PSEC
to fully correct for the policy sampling error but not transition dynamics sampling error.
Figure \ref{fig:gridworld_stoch} confirms that PSEC
is achieves a lower final MSVE than TD as $p \to 1$.
As $p \to 0$, the transition dynamics 
become the dominant source of sampling error and PSEC-TD(0) and TD(0) perform similarly.

\subsection{Function Approximation: Continuous States and Discrete Actions (CartPole)}
\label{app:cartpole_exps}

In each experiment below, unless stated, the following components were fixed: a batch size of $10$ episodes, the value function was represented
with a neural network of single hidden layer of $512$ neurons using tanh activation, the gradients were normalized to unit norm 
before the gradient descent step was performed, we used a learning rate of $1.0$ and decayed the learning rate by $10\%$ every
$50$ presentations of the batch to the algorithm. The true MSVE was computed by $200$ Monte Carlo rollouts for
$150$ sampled states.

\subsubsection{Data Efficiency}

From Figure \ref{fig:exp_cart_msve_batch}. Both algorithms used a learning rate of $1.0$ and decayed
the learning rate by $5\%$ every $10$ presentations of the batch. The PSEC model architecture was a neural
network with $3$ hidden layers with $16$ neurons each. Batch sizes of $10$, $50$, and $500$ episodes
used a learning rate of $0.0125$, batch size of $100$ used $0.003125$, and batch size of $1000$ used
$0.001563$.

\subsubsection{Effect of Value Function Model Architecture}

From Figure \ref{fig:exp_psec_vs_vf}. PSEC used a model architecture of $3$ hidden layers with $16$ neurons each and tanh activation, and learning
rate of $0.025$. 

\subsubsection{Effect of PSEC Learning Rate}

\begin{figure}[]
    \centering
    \includegraphics[scale=0.16]{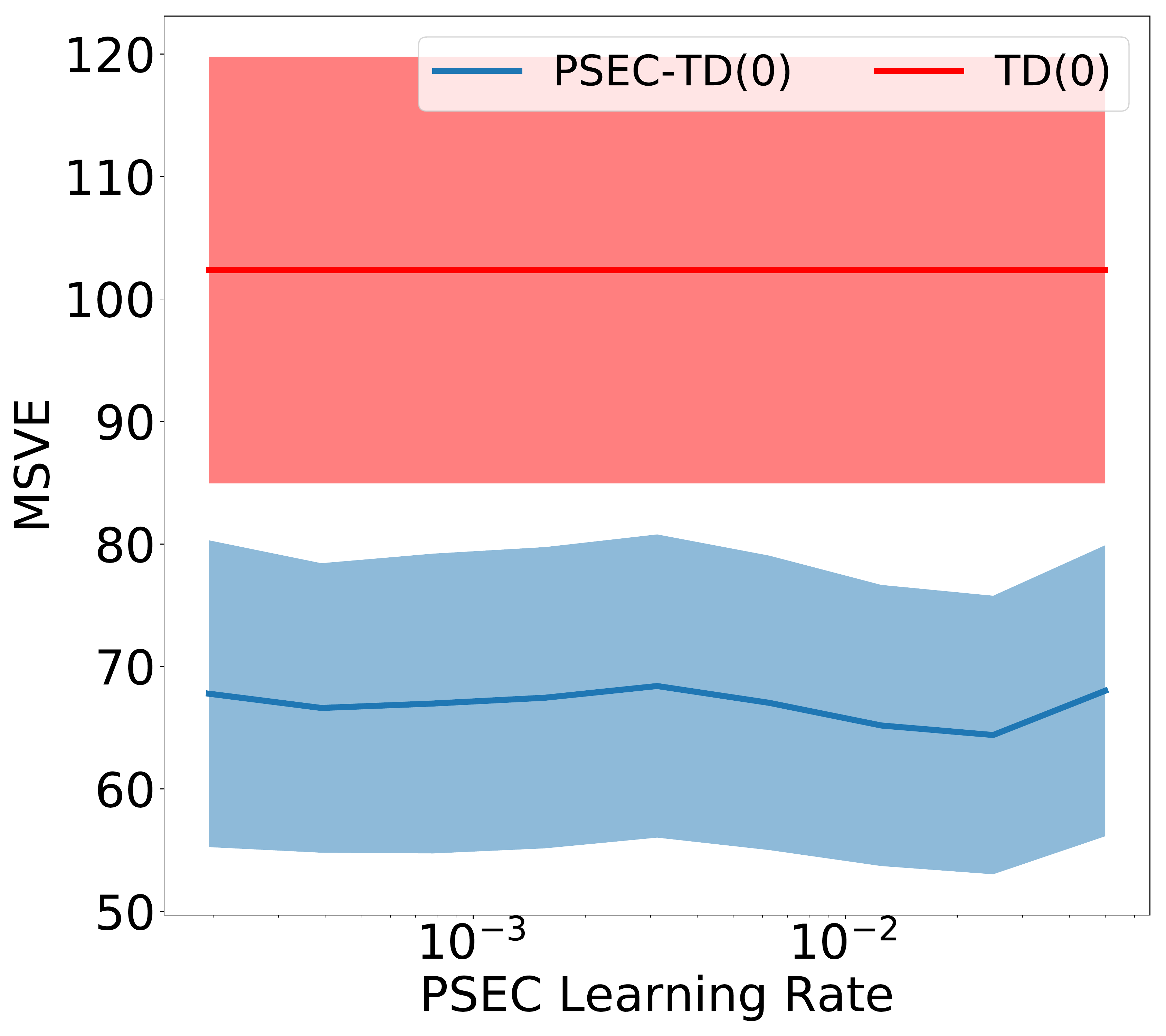}
    \caption{\footnotesize Compares data efficiency of PSEC with varying learning rates against TD respectively on CartPole. Results use a batch size of $10$ episodes, and
    results are averaged over $300$ trials with $95\%$ confidence error bars.}
    \label{fig:exp_psec_vs_lr}
\end{figure}

Figure \ref{fig:exp_psec_vs_lr} compares the data efficiency of PSEC-TD vs TD for varying learning rates of the PSEC
policy, while holding the value function, PSEC model, and behavior policy
architectures fixed, on CartPole. Since TD does not use PSEC,
its error for a given batch size is independent of the PSEC learning rate. From above, PSEC appears
to be relatively stable in its improvement over TD regardless of the learning rate used. The PSEC policy used in this experiment was a neural network with: $3$  hidden layers with $16$ neurons each and tanh activation. 

\subsubsection{Effect of PSEC Model Architecture}
From Figure \ref{fig:exp_psec_arch}. PSEC used a learning
rate of $0.025$. The chosen PSEC neural network models are with respect to the behavior policy described earlier, a 
$2$ hidden layered with $16$ neurons architecture.

\subsubsection{Varying PSEC Training Style}

\begin{figure}[H]
    \centering
    \includegraphics[scale=0.4]{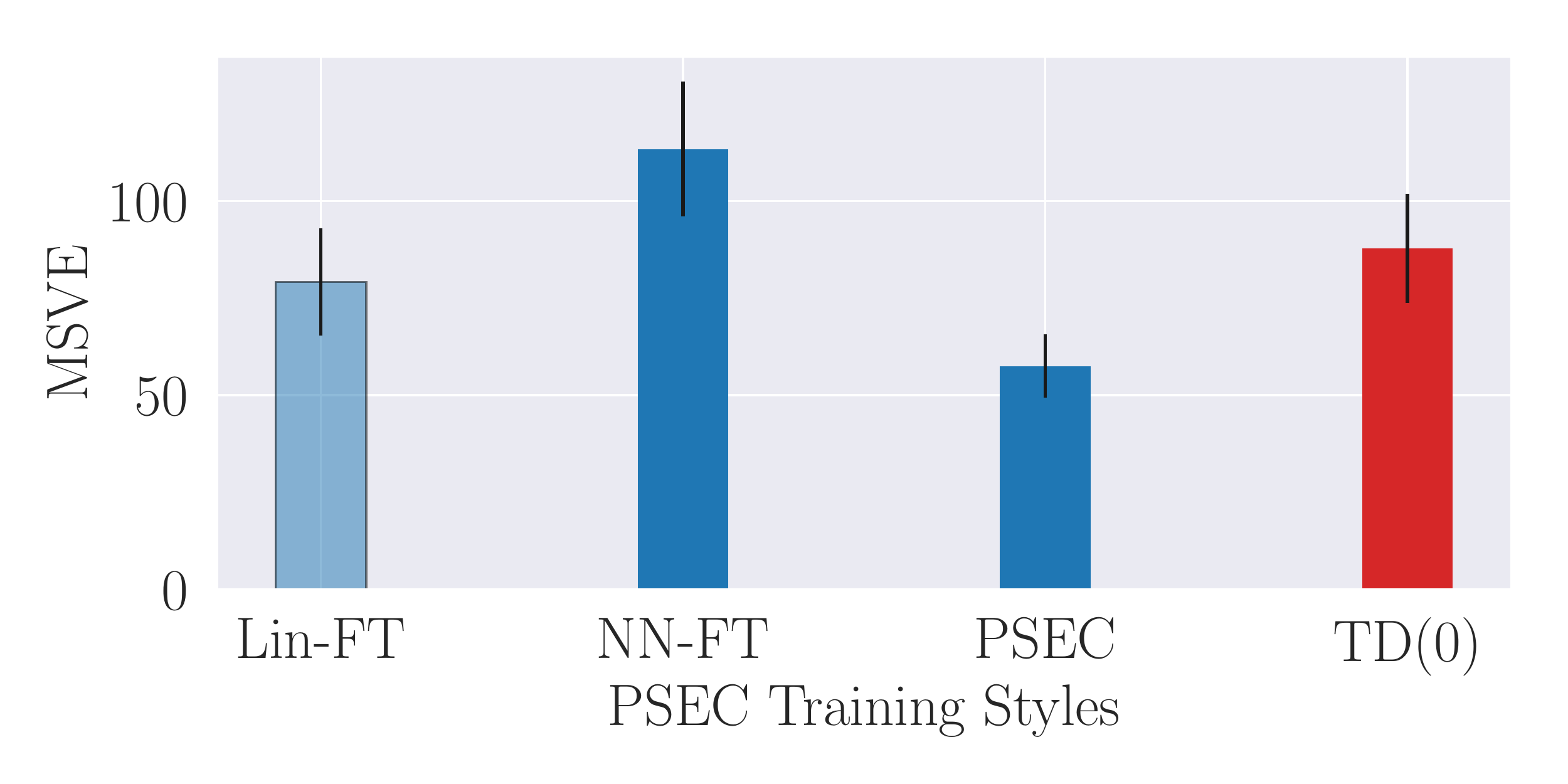}
    \caption{Comparing data efficiency of varying training styles of PSEC against TD for a fixed batch size of $10$ episodes. 
    Results shown are averaged over $300$ trials
    and shaded region is $95\%$ confidence. Darker shades represent statistically significant result.}
    \label{fig:exp_varying_spec}
\end{figure}

Figure \ref{fig:exp_varying_spec} illustrates the data efficiency of three variants of PSEC, while holding the
value-function PSEC model, and behavior policy architectures fixed. All three variants
use the same PSEC model architecture as that of the behavior policy, and each used a learning rate of $0.025$ with
tanh activation. The
three variants are as follows: 1) Lin-FT is when PSEC initializes the PSEC model to the weights of the behavior policy
and trains on the batch of data by finetuning only the last linear layer, 2) NN-FT is when PSEC initializes the
PSEC model to the weights of the behavior policy but finetunes the all the weights of the network, and 3) PSEC uses
the same
training style in the previous experiments, where the model is initialized randomly and all the weights are
tuned. We found that Lin-FT performed similarly to TD with a statistically insignificant improvement over TD; we believe this may be so since Lin-FT is initialized to 
the behavior policy and since there are only few weights to change in the linear layer, the newly learned
Lin-FT is still similar to the behavior policy, which would produce PSEC corrections close to 1 (equivalent to TD). Interestingly, tuning all the weights of the neural network did better when the model
was initialized randomly versus when it was initialized to the behavior policy.

\subsubsection{Effect of Underfitting and Overfitting during PSEC Policy Training}

This experiment attempts to give an understanding of how the MSVE achieved by each PSEC variant is
dependent on the number of epochs the PSEC model was trained for, while holding the value-function, PSEC
and behavior policy architectures fixed. We conduct the experiment as follows:
the PSEC algorithm performs $10$ gradient descent steps (epochs) on the full batch of data, after which the resulting training
and validation mean cross-entropy losses are plotted along with the MSVE achieved by that trained PSEC
policy. For example, after $10$ epochs, the training and validation loss of the PSEC model was nearly $0.5$
and the model achieved an MSVE of nearly $150$.

\begin{figure}[H]
    \centering
        \subfigure{\includegraphics[scale=0.2]{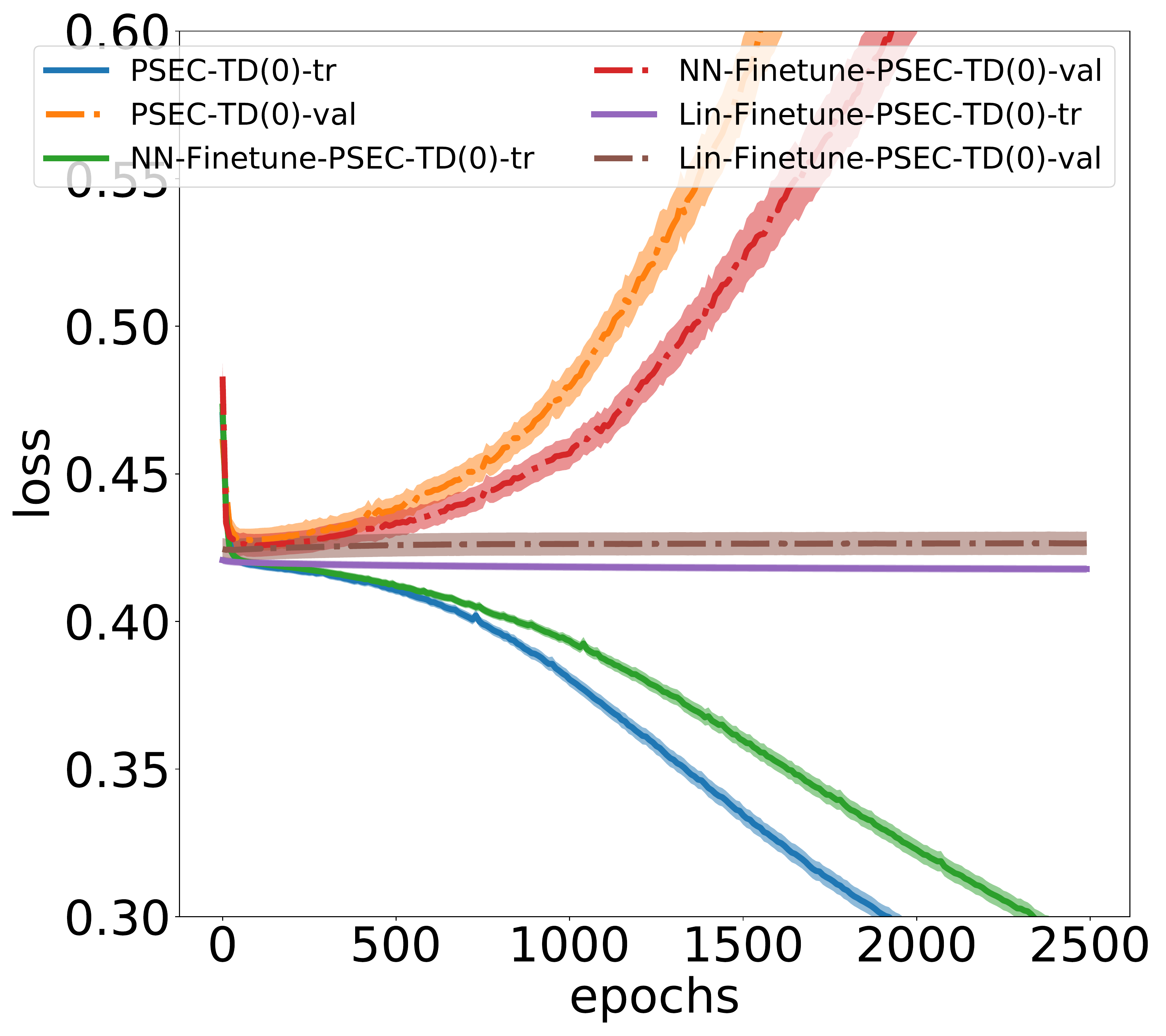}}
        \subfigure{\includegraphics[scale=0.2]{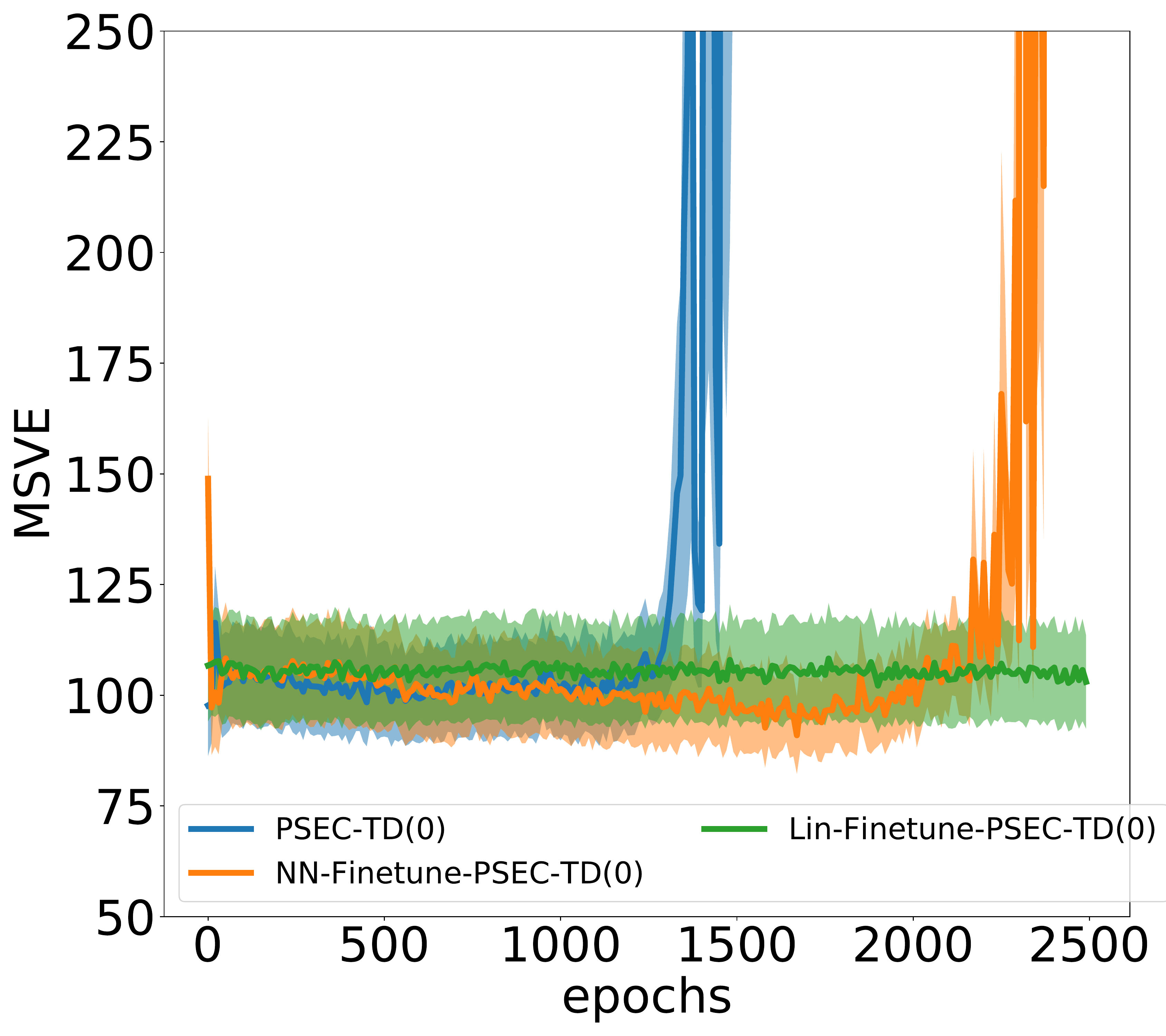}}
\caption{Comparing MSVE achieved and cross entropy loss (training (tr) and validation (val)) by variants of PSEC after each epoch of training for a fixed batch size of $10$ episodes. Results shown are averaged over $50$ trials
    and shaded region is $95\%$ confidence.}
	\label{fig:exp_psec_overfitting}
\end{figure}

Since computing the MSVE can be computationally expensive, as it requires processing the batch until the value function
converges, we change the learning rate decay schedule to starting with
a learning rate of $1.0$ but decaying learning rate by $50\%$ every $50$ presentations of the full batch
to the algorithm (this change is also the
reason why these results may be different from the ones shown earlier). All PSEC variants used a learning rate of $0.025$ and PSEC model architecture of $2$ hidden layers
with $16$ neurons each and tanh activation.

Figure \ref{fig:exp_psec_overfitting} suggests that performance of PSEC, regardless of the variant, depends on the
number of epochs it was trained for. Naturally, we do want to fit sufficiently well to the data, and the graph suggests
that some overfitting is tolerable. However, if overfitting becomes extreme, PSEC's performance suffers, resulting
in MSVE nearly $1000$ times larger than the minimum error achieved (not shown for clarity). From the graph, we
can see that the PSEC variant, which is initialized randomly, starts to extremely overfit before the NN-finetune variant
does, causing its MSVE to degrade before that of NN-finetune variant. We also see that the Lin-finetune variant
is not able to overfit since the last linear layer may not be expressible enough to overfit, causing it to have a relatively
stable MSVE across all epochs.

\subsubsection{Effect of Behavior Policy Distribution}
So far, PSEC has used a function approximator of the similar function class as that of the behavior policy i.e.
both the PSEC models and the behavior policy were neural networks of similar architectures. In this experiment, we evaluate
the performance of PSEC with a neural network policy when the behavior policy that models a discontinuous function 
generates a larger batch size
of $100$ episodes, while the value function and PSEC model architectures are fixed. In particular, we use
a behavior policy in CartPole that does the following: if the sign of the pole angle is negative, move left with probability
$0.75$ and right with probability $0.25$, and if the sign of the pole angle is positive, move right with probability $0.75$ and left with
probability $0.25$. The PSEC policy is a neural network with $3$ hidden layers with $16$ neurons each with tanh activation.
\begin{figure}[H]
    \centering
    \includegraphics[scale=0.5]{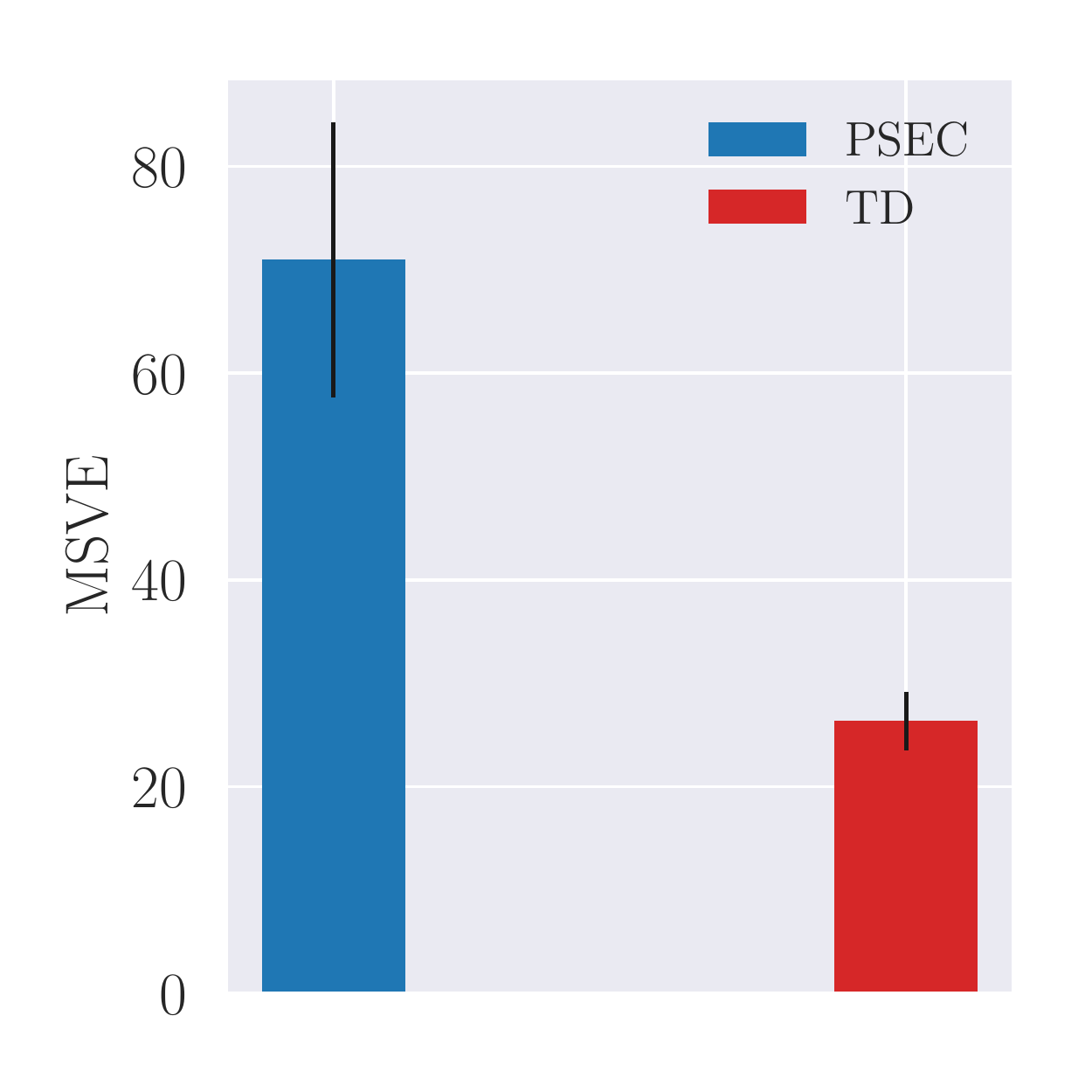}
    \caption{Comparing data efficiency of PSEC against TD for a fixed batch size of $100$ episodes when the
    behavior policy models a discontinuous function.
    Results shown are averaged over $30$ trials
    and shaded region is $95\%$ confidence.}
    \label{fig:exp_hard_policy}
\end{figure}

Figure \ref{fig:exp_hard_policy} shows that PSEC performs much worse than TD when the behavior policy is 
the discontinuous type function described above. We reason that the neural network finds it difficult to compute
the MLE of the data since this discontinuous distribution is ``hard" to model; therefore, producing incorrect 
PSEC weights, which degrade its  performance. While we can largely ignore the distribution of the behavior policy,
this experiment shows that PSEC may suffer in situations like the one described.

\subsection{Function Approximation: Continuous States and Actions (InvertedPendulum)}
\label{app:inverted_pend_exps}

In each experiment below, unless stated, the following components were fixed: a batch size of $20$ episodes, the value function was represented
with a neural network of $2$ hidden layer with $64$ neurons each using tanh activation, 
the gradients were normalized to unit norm 
before the gradient descent step was performed, we used a learning rate of $1.0$ and decayed the learning rate by $5\%$ every
$10$ presentations of the batch to the algorithm. The true MSVE was computed by $100$ Monte Carlo rollouts for
$100$ sampled states.

\subsubsection{Data Efficiency}

From Figure \ref{fig:exp_invpen_msve_batch}. The PSEC model architecture was a neural
network with $2$ hidden layers with $64$ neurons each. Batch
sizes $10, 50, 100, 500, 1000$ used a learning rate of $0.000781$. Batch
size of $10$ used an L$2$ weight penalization
of $0.02$. All used a value
function model architecture of $3$ hidden layers with $64$
neurons each.

\subsubsection{Effect of Value Function Model Architecture}
\begin{figure}[H]
    \centering
    \subfigure[InvertedPendulum]{\label{fig:exp_psec_vs_vf_cont_act}\includegraphics[scale=0.35]{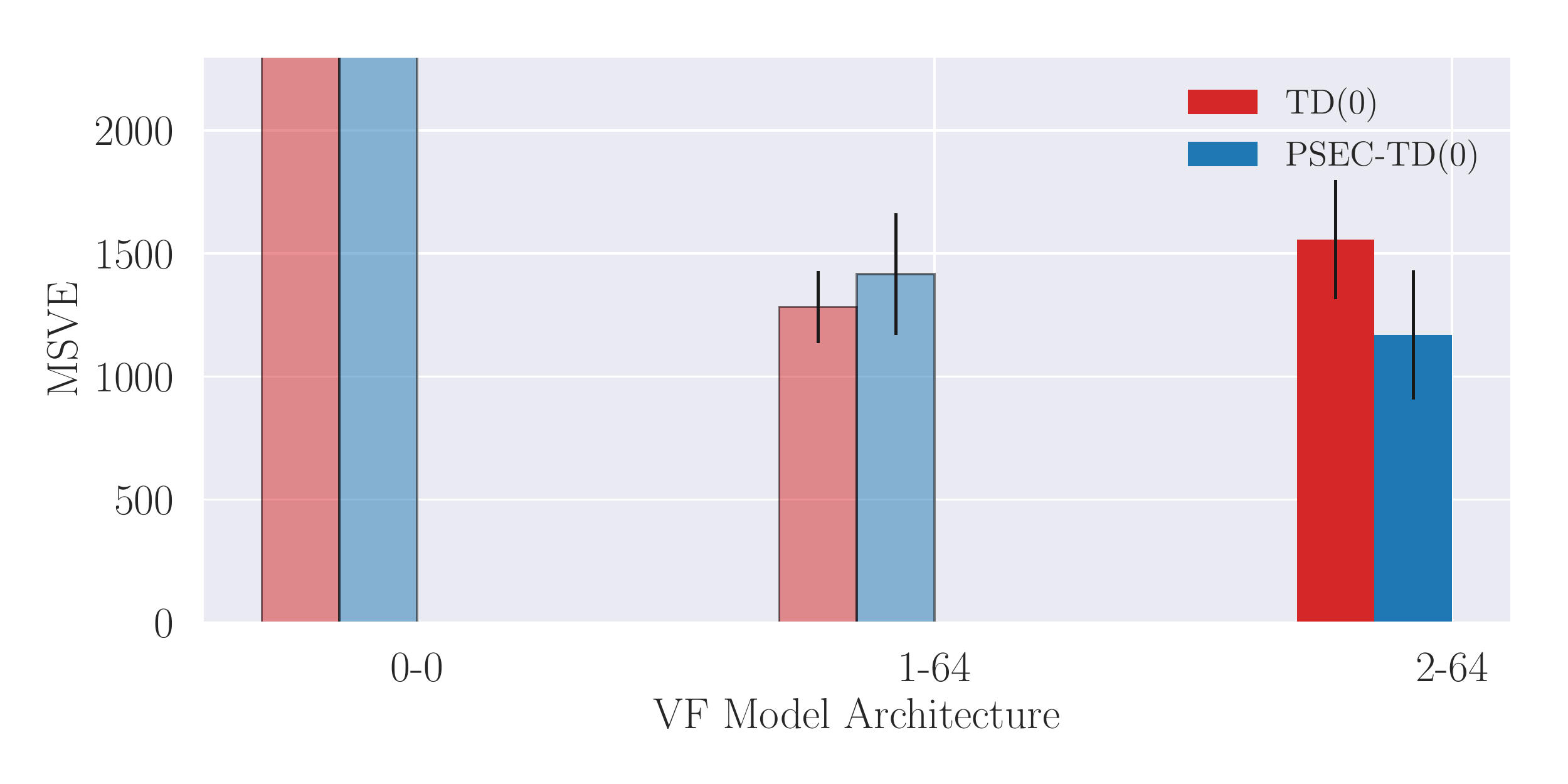}}
    \caption{\footnotesize Comparing data efficiency of PSEC with varying VF model architectures against TD. Figure \ref{fig:exp_psec_vs_vf_cont_act} use batch size of $20$ episodes , and results shown are averaged over $350$ trials respectively with error bars of $95\%$ confidence. Darker shades represent statistically significant result. The label on
    the $x$ axis shown is ($\#$ hidden layers - $\#$ neurons).}
\end{figure}

From Figure \ref{fig:exp_psec_vs_vf_cont_act}. The PSEC policy is
 $2$ hidden layers with $64$ neurons each and used a learning rate of $0.000781$.

\subsubsection{Effect of PSEC Learning Rate}
\begin{figure}[H]
    \centering
    \subfigure[]{\includegraphics[scale=0.18]{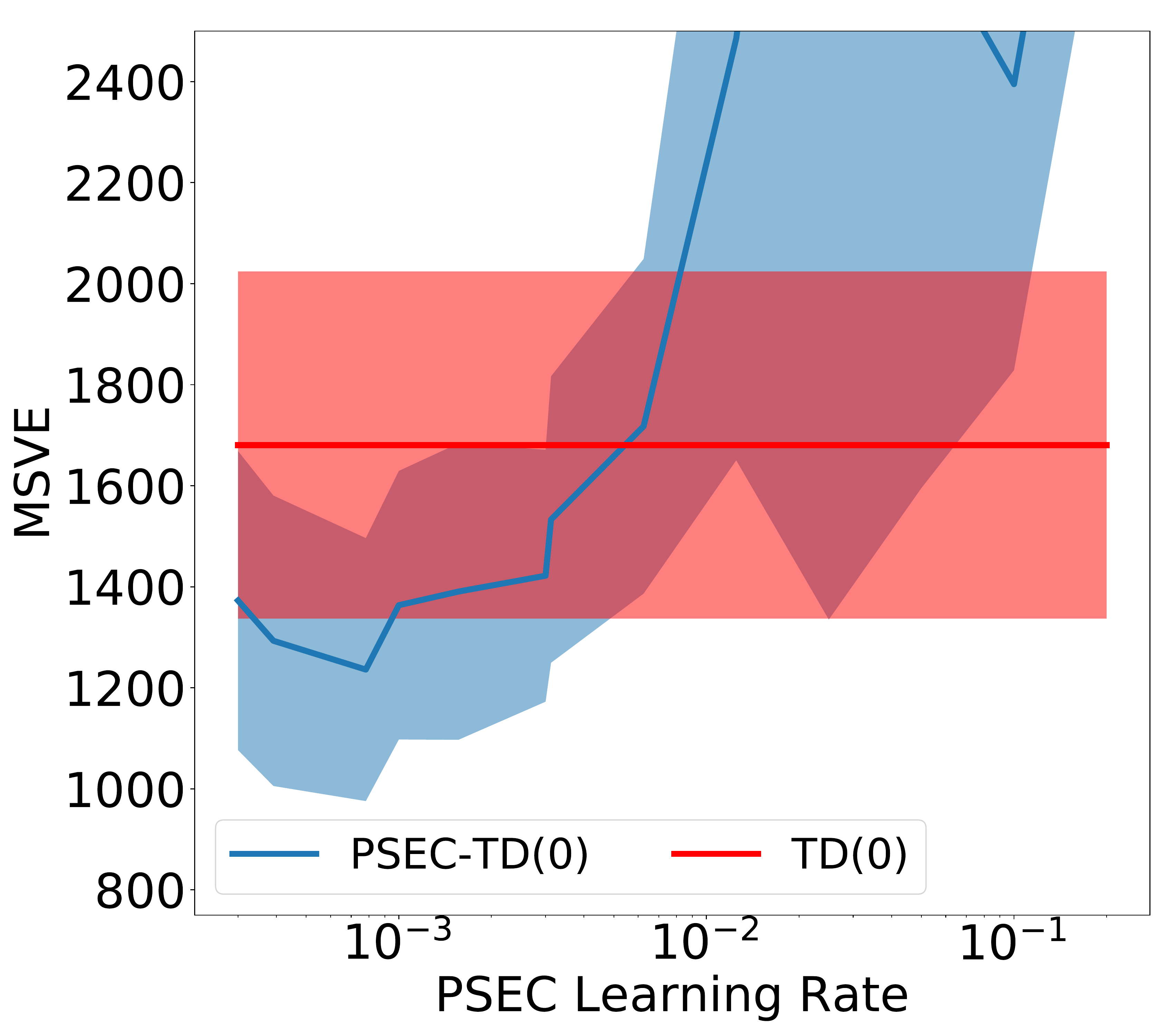}}
    \caption{Comparing data efficiency of PSEC with varying learning rates against TD for a fixed batch size of $20$ episodes. 
    Results shown are averaged over $200$ trials
    and error bar is $95\%$ confidence.}
    \label{fig:exp_psec_vs_lr_cont_act}
\end{figure}

Figure \ref{fig:exp_psec_vs_lr_cont_act} compares the data efficiency of PSEC-TD to TD with varying learning rates
for PSEC, while holding the PSEC and value function architecture fixed. Unlike earlier, the PSEC learning rate heavily
influences the learned value function in the continuous state and action setting. In general, PSEC performance heavily
degraded when the learning rate increased ($y$-axis limited for clarity). Among the tested learning rates, $0.000781$ was the optimal, giving
a statistically significant result.

\subsubsection{Effect of PSEC Model Architecture}

\begin{figure}[H]
    \centering
    \includegraphics[scale=0.5]{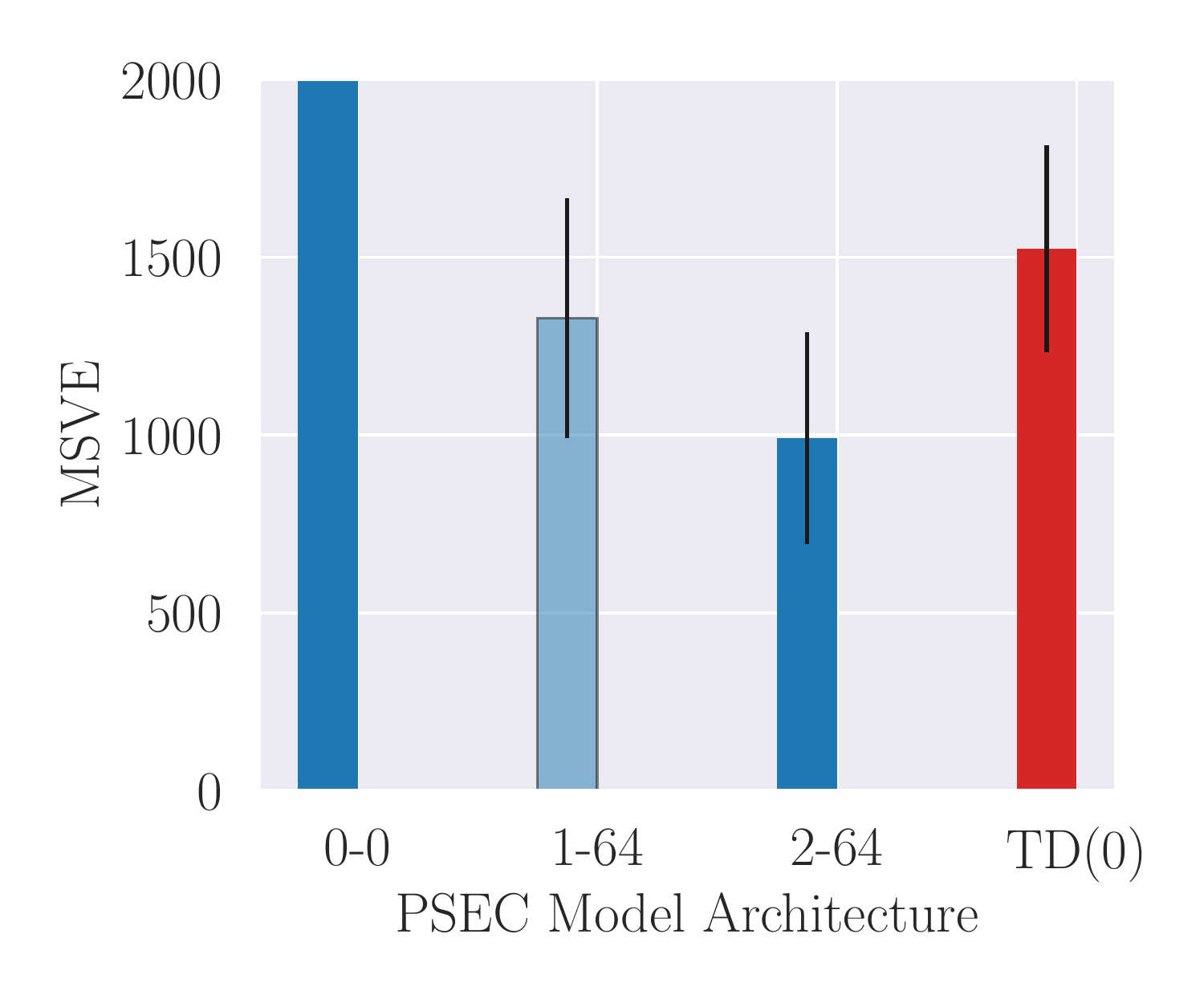}
    \caption{Comparing data efficiency of PSEC with varying model architectures against TD for a fixed batch size of $20$ episodes. 
    Results shown are averaged over $200$ trials
    and error bar is $95\%$ confidence. Darker shades represent statistically significant result. The label on
    the $x$ axis shown is ($\#$ hidden layers - $\#$ neurons). The value function represented by $0$-$0$ is
    a linear mapping with no activation function.}
    \label{fig:exp_psec_arch_cont_act}
\end{figure}

Figure \ref{fig:exp_psec_arch_cont_act} compares the data efficiency of PSEC-TD with TD with varying PSEC
model architectures, while holding the value-function and behavior policy architectures fixed. All the shown
PSEC architectures used a learning rate of learning rate $0.000781$.  Similar to our earlier findings,
a more expressive network was able to better model the batch of data and produce a statistically significant
improvement over TD. Less expressive PSEC models performed worse than TD, and any improvement
was statistically insignificant. Note that the linear architecture used produced an MSVE of $\sim 5800$ (not shown for clarity) 
and its poor data efficiency with respect to TD(0) was statistically significant. 

\subsubsection{Effect of Underfitting and Overfitting during PSEC Policy Training}

\begin{figure}[H]
    \centering
        \subfigure{\includegraphics[scale=0.15]{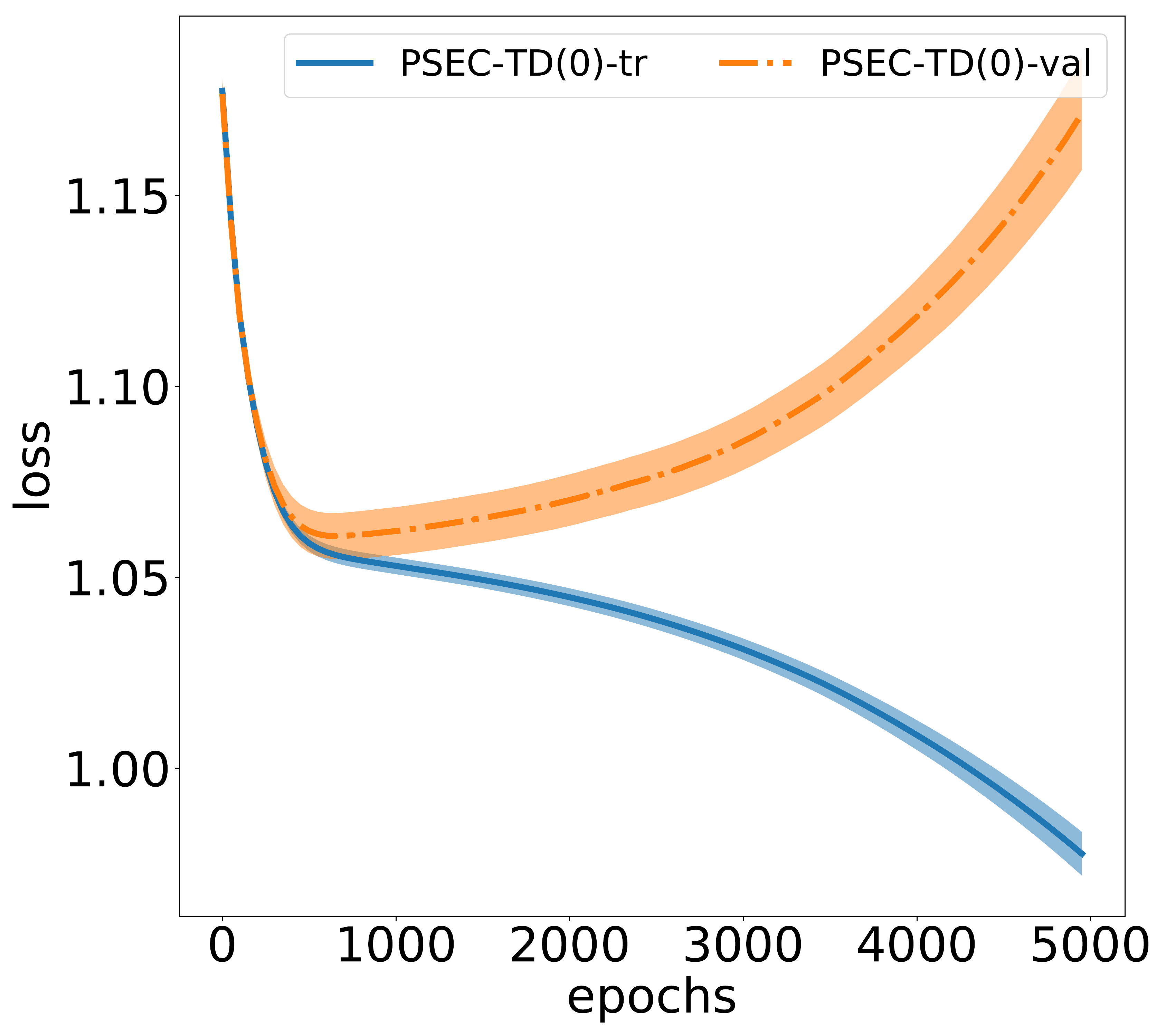}}
        \subfigure{\includegraphics[scale=0.15]{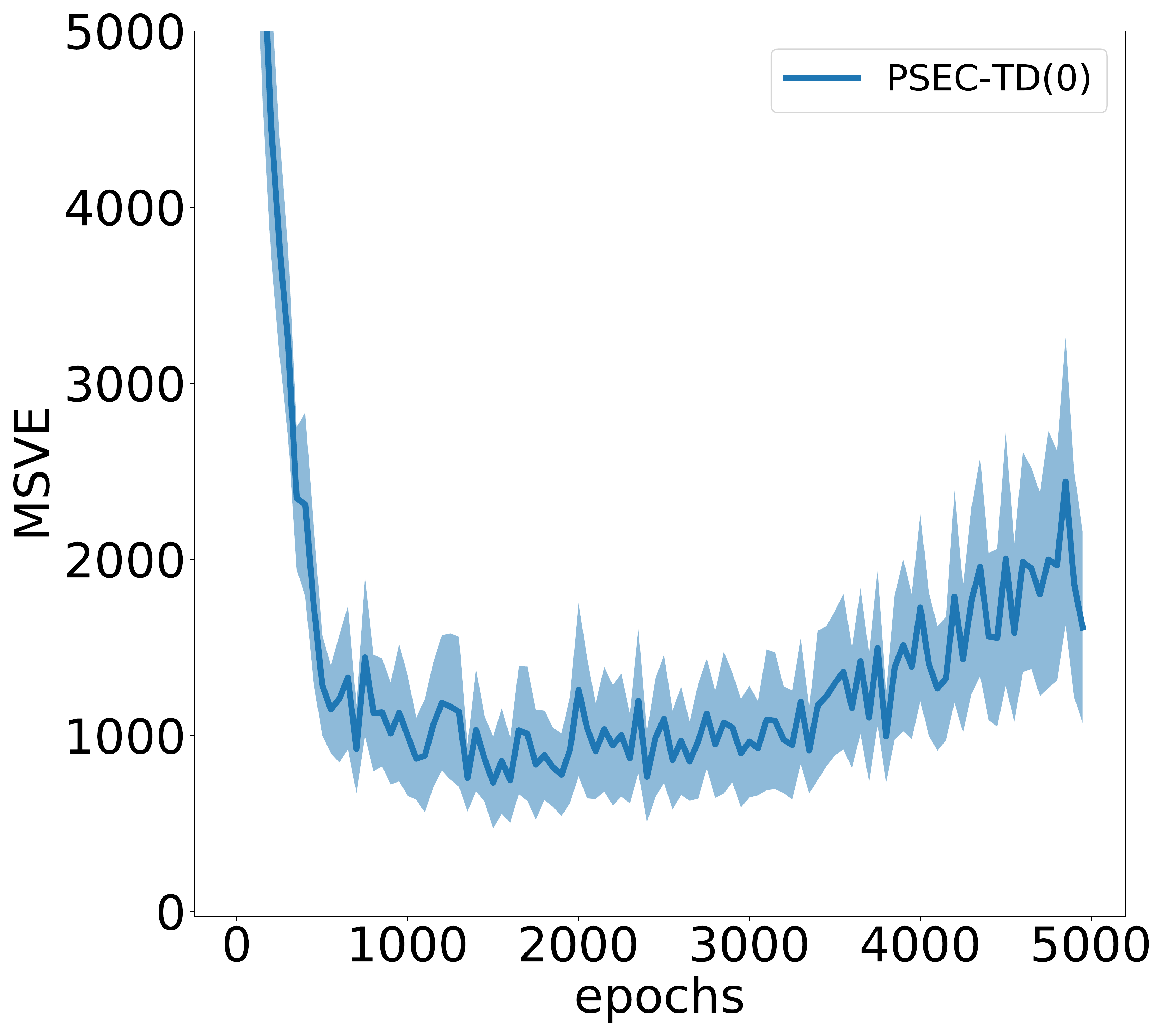}}
\caption{\footnotesize Comparing MSVE achieved, and training (tr)
and validation (val) loss by PSEC after each epoch of training for a fixed batch size of $20$ episodes on InvertedPendulum. Results shown are averaged over $100$ trials
    and shaded region is $95\%$ confidence.}
	\label{fig:exp_psec_overfitting_cont_act}
\end{figure}

This experiment attempts to give an understanding of how the MSVE achieved by each PSEC variant is
dependent on the number of epochs the PSEC model was trained for, while holding the value-function, PSEC
and behavior policy architectures fixed on InvertedPendulum. The experiment is conducted in a similar manner as before except we perform
$50$ gradient descent steps  (epochs) before plotting the training and validation loss, and MSVE achieved by PSEC after the 
gradient steps. The loss shown here is a regression loss detailed in Appendix \ref{app:invertedpendulum}. Figure \ref{fig:exp_psec_overfitting_cont_act} suggests that performance of PSEC depends on the
number of epochs it was trained for. Naturally, we do want to fit sufficiently well to the data, and the graph suggests
that some overfitting is tolerable. However, if overfitting becomes extreme, PSEC's performance suffers.  If some overfitting is desirable, then early stopping
is not the preferred principled approach to terminate PSEC model training.  When computing MSVE, we used the learning rate schedule specified at the beginning of this section. PSEC used a learning rate of $0.000781$ and model architecture of $2$ hidden layers
with $64$ neurons each and tanh activation.

\section{Extended Empirical Description}
\label{app:exp_details}

In this appendix we provide additional details for our empirical evaluation.

\subsection{Gridworld}
\label{app:gridoworld}
\begin{figure}[!h]
    \centering
    \includegraphics{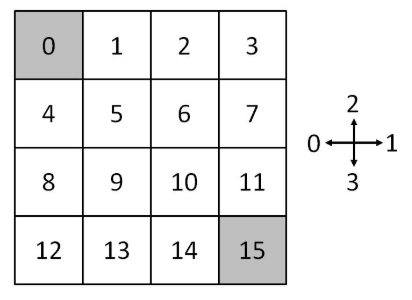}
    \caption{The Gridworld environment. Start at top left, bottom right is terminal state, discrete action space consists of the cardinal directions,  and
    discrete state space is the location in the grid. This specific image was taken from \href{https://cran.csiro.au/web/packages/reinforcelearn/vignettes/environments.html}{this link}.}
    \label{fig:gridworld}
\end{figure}

This domain is a $4 \times 4$ grid, where an agent starts at $(0, 0)$
and tries to navigate to $(3,3)$. The states are the discrete positions in
the grid and actions are the $4$ cardinal directions. The reward function 
is $100$ for reaching $(3,3)$, $-10$ for reaching $(1,1)$, $1$ for
reaching $(1,3)$, and $-1$ for reaching all other states. If an agent takes
an action that hits a wall, the agent stays in the same location. The
transition dynamics are controlled by a parameter, $p$, where with probability
$p$, an agent takes the intended action, else  it takes
an adjacent action with probability $(1 - p)/2$.  All policies use a softmax action selection distribution with value $\theta_{sa}$, for each state, s, and action a. The probability of
taking action a in state s is given by:

\begin{equation*}
    \pi(a|s) = \frac{e^{\theta_{sa}}}{\sum_{a' \in \mathcal{A}}e^{\theta_{sa'}}}
\end{equation*}

In the on-policy experiments, the evaluation and behavior
policies were equiprobable policies in each cardinal direction.
In the off-policy experiments, the evaluation policy was
such that each $\theta$ was generated
from a standard normal distribution and behavior policy was the
equiprobable policy.

For the comparisons of batch linear PSEC-TD(0) and TD(0), we
conducted a parameter sweep of the learning rates for the
varying batch sizes. The parameter sweep was over: 
$\{5e^{-3}, 1e^{-3}, 5e^{-2}, 1e^{-2}, 5e^{-1}\}$. We used
a value function convergence threshold of $1e^{-10}$. For PSEC-LSTD and LSTD, we stabilized the matrix, $A$,
before inverting it by adding $\epsilon I$ to the
computed $A$. We conducted a parameter sweep over
the following: 
$\epsilon \in \{1e^{-6}, 1e^{-5}, 1e^{-4}, 1e^{-3}, 1e^{-2}, 1e^{-1}\}$. 

\subsection{CartPole}

\begin{figure}[H]
    \centering
    \includegraphics[scale=0.3]{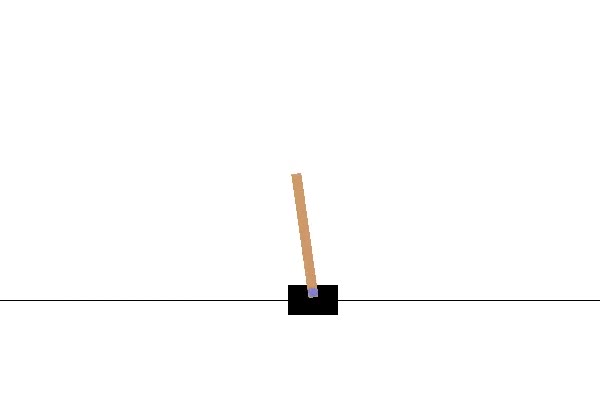}
    \caption{CartPole-v0 from OpenAI Gym \citep{Brockman2016OpenAIG}}
    \label{fig:cartpole}
\end{figure}

In this domain, the goal of the agent is to balance a pole for as long as possible. We trained our behavior policy
using REINFORCE \citep{REINFORCE} with the Adam optimizer \citep{adam} with learning rate $3e^{-4}$, $\beta_1 = 0.9$,
and $\beta_2 = 0.999$. The behavior policy mapped raw state features to a softmax distribution over actions. The policy was a neural network with $2$ hidden layers with $16$ neurons each, and used
the tanh activation function and was initialized with Xavier initialization \citep{xavier}. 

The value function used by all algorithms was initialized by Xavier initialization and used the tanh activation function, and was trained using semi-gradient TD \citep{Sutton:1998:IRL:551283}. We used a convergence threshold
of $0.1$.

The PSEC policy was initialized by Xavier initialization and used the tanh activation function. PSEC-TD sweeped
over the following learning rates $\alpha \in \{0.1 \times 2.0^j
| j = -7, -6, ... 1, 2\}$. It used a validation set of $10\%$ the size of the batch size. It used an L$2$ regularization of
$2e^{-2}$.  More details can be found in Section \ref{sec:exp} and Appendix \ref{app:cartpole_exps}.

\subsection{InvertedPendulum}
\label{app:invertedpendulum}

\begin{figure}[!h]
    \centering
    \includegraphics[scale=0.3]{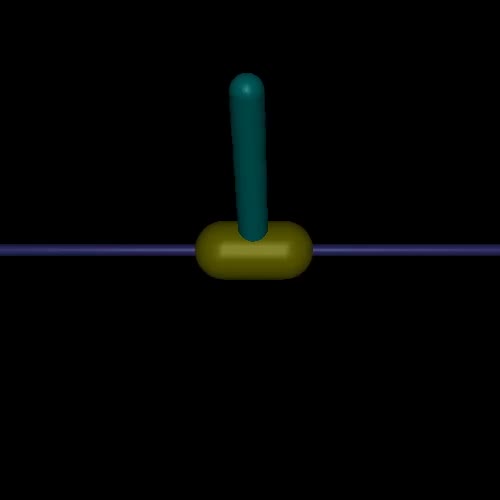}
    \caption{InvertedPendulum-v2 from OpenAI Gym and MuJoCo \citep{Brockman2016OpenAIG, mujoco}}
    \label{fig:invertedpendulum}
\end{figure}

In this domain, the goal of the agent is to balance a pole for as long as possible. We trained our behavior policy
using PPO \citep{ppo} with the default settings found on Gym \citep{Brockman2016OpenAIG}. The policy was a neural network with $2$ hidden layers with $64$ neurons each, and used
the tanh activation function and was initialized with Xavier initialization \citep{xavier}. It mapped state features to
an output vector that represented the mean vector of a Gaussian distribution. This mapping along with a separate
parameter set representing the log standard deviation of each element in the output vector, make up the policy. The policy
was trained by minimizing the following loss function:

\begin{align*}
\mathcal{L} = \sum_{i=1}^m 0.5 ((a_i - \mu(s_i))/e^\sigma)^2 + \sigma
\end{align*}

where $m$ are the number of state-action training examples, $a_i$ is the action vector of the $i^{th}$ example,
$\mu(s_i)$ is the mean vector outputted by the neural network of the Gaussian distribution for state $s_i$, and 
$\sigma$ is the the seperate parameter representing the log standard deviation of each element in the output vector,
$\mu(s_i)$.

The value function used by all algorithms was initialized by Xavier initialization and used the tanh activation function,
and was trained using semi-gradient TD \citep{Sutton:1998:IRL:551283}. We used a convergence threshold
of $0.1$.

The PSEC policy was initialized by Xavier initialization and used the tanh activation function. PSEC-TD swept
over the following learning rates $\alpha \in \{0.1 \times 2.0^j
| j = -8, -6, ..., 1, 2\}$. It used a validation set of $20\%$ the size of the batch size. More details can be found in
Section  \ref{sec:exp} and Appendix \ref{app:cartpole_exps}.
\end{document}